\documentclass[10pt,journal,compsoc]{IEEEtran}
\ifCLASSOPTIONcompsoc
  \usepackage[nocompress]{cite}
\else
  \usepackage{cite}
\fi
\ifCLASSINFOpdf
\else
\fi

\usepackage[english]{babel}
\usepackage{amsmath}
\usepackage{amsfonts}
\usepackage{array}
\usepackage{graphicx}
\usepackage{bm}
\usepackage{fixmath}
\usepackage{blindtext}
\usepackage{amsthm}
\usepackage{mathrsfs}
\usepackage{enumerate}
\usepackage{amsthm}
\usepackage{float}\usepackage{nomencl}\usepackage{multirow}
\usepackage{bm}
\usepackage{lipsum}
\usepackage{amssymb}
\usepackage{algorithmic}
\usepackage{epstopdf}
\usepackage{caption}
\usepackage{subcaption}\usepackage{xcolor}
 \usepackage{url}\usepackage{lettrine}
 \usepackage{balance}
\usepackage{enumerate}
\usepackage{amsmath}\usepackage{amssymb}
\usepackage{amsthm}
\usepackage{algorithm}
\usepackage{graphicx}
\usepackage{float}
\usepackage{bm} 
\usepackage{epstopdf}  

\usepackage{doi}
\usepackage{xcolor}
\usepackage{mathtools}
\usepackage{multirow}
\usepackage{bm}
\usepackage{makecell}

\usepackage{mathrsfs} \usepackage{upgreek}
\usepackage{mathrsfs}
\usepackage{enumerate}
\usepackage{amsmath}
\usepackage{amsthm}
\usepackage{caption}
\usepackage{graphicx}\usepackage{lettrine}
\usepackage{float}
\usepackage{bm}\usepackage{array}
\usepackage{algorithm}
\usepackage{lipsum}
\usepackage{subcaption}
\usepackage{amssymb}
\usepackage{epstopdf} 
\usepackage{nomencl}\usepackage{stfloats}
\usepackage{url}\usepackage{booktabs}\usepackage{colortbl}\usepackage{epstopdf}
\usepackage{etoolbox}
\def\BibTeX{{\rm B\kern-.05em{\sc i\kern-.025em b}\kern-.08em
    T\kern-.1667em\lower.7ex\hbox{E}\kern-.125emX}}

\usepackage{nicefrac}       
\usepackage{microtype}      
\usepackage{lipsum}         
\usepackage{doi}
\usepackage{xcolor}
\usepackage{mathtools}
\usepackage{multirow}
\usepackage{bm}
\usepackage{makecell}

\usepackage{array}
\usepackage{textcomp}
\usepackage{stfloats}
\usepackage{url}
\usepackage{verbatim}

\newtheorem{theorem}{Theorem}
\newtheorem{lemma}{Lemma}
\newtheorem{definition}{Definition}
\newtheorem{remark}{Remark}

\newtheorem{proposition}{Proposition}

\allowdisplaybreaks[4]

\title{Feature Matching Data Synthesis for  Non-IID Federated Learning 
}
\author{\IEEEauthorblockN{Zijian Li,~\IEEEmembership{Graduate Student Member, IEEE}}, \IEEEauthorblockN{Yuchang Sun,~\IEEEmembership{Graduate Student Member, IEEE}},\\
\IEEEauthorblockN{Jiawei Shao,~\IEEEmembership{Graduate Student Member, IEEE}},
\IEEEauthorblockN{Yuyi Mao,~\IEEEmembership{Member, IEEE}}, \\
\IEEEauthorblockN{Jessie Hui Wang,~\IEEEmembership{Member, IEEE}}, and
\IEEEauthorblockN{Jun Zhang,~\IEEEmembership{Fellow, IEEE}}
\thanks{Z. Li, Y. Sun, J. Shao, and J. Zhang are with the Department of Electronic and Computer Engineering, The Hong Kong University of Science and Technology,
Hong Kong, China (E-mail: \{zijian.li, yuchang.sun, jiawei.shao\}@connect.ust.hk, eejzhang@ust.hk).
Y. Mao is with the Department of Electrical and Electronic,
The Hong Kong Polytechnic University, Hong Kong, China (E-mail: yuyi-eie.mao@polyu.edu.hk). 
Jessie Hui Wang is with the Institute for Network Sciences and Cyberspace,
Tsinghua University, Beijing 100084, China, and also with ZGC Lab, Beijing
100194, China (e-mail: jessiewang@tsinghua.edu.cn).
(Corresponding author: Yuyi Mao.)}}

\begin{document}


\IEEEtitleabstractindextext{%
\begin{abstract}
Federated learning (FL) has emerged as a privacy-preserving paradigm that trains neural networks on edge devices without collecting data at a central server. 
However, FL encounters an inherent challenge in dealing with non-independent and identically distributed (non-IID) data among devices.
To address this challenge, this paper proposes a hard feature matching data synthesis (HFMDS) method to share auxiliary data besides local models.
Specifically, synthetic data are generated by learning the essential class-relevant features of real samples and discarding the redundant features, which helps to effectively tackle the non-IID issue.
For better privacy preservation, we propose a hard feature augmentation method to transfer real features towards the decision boundary, with which the synthetic data not only improve the model generalization but also erase the information of real features.
By integrating the proposed HFMDS method with FL, we present a novel FL framework with data augmentation to relieve data heterogeneity.
The theoretical analysis highlights the effectiveness of our proposed data synthesis method in solving the non-IID challenge. 
Simulation results further demonstrate that our proposed HFMDS-FL algorithm outperforms the baselines in terms of accuracy, privacy preservation, and computational cost on various benchmark datasets.
\end{abstract}

 }
\maketitle


%
%
%
%
\IEEEraisesectionheading{\section{Introduction}\label{sec:introduction}}
\IEEEPARstart{T}{he}
development of deep learning (DL) has revolutionized the field of computer vision, enabling the creation of more accurate and sophisticated AI applications in edge devices, such as autonomous driving \cite{auto_driving}, virtual reality (VR) services \cite{9741351}, and unmanned aerial vehicle control \cite{qi2022task}.
However, the success of DL depends on the quantity and quality of training data, and the collection and analysis of data pose a severe risk to individual privacy.
To prevent the privacy risks associated with data collection and analysis, federated learning (FL) has recently emerged as a promising solution that exploits distributed data and computational resources to collaboratively train a global DL model \cite{federated_survey_1}.
In particular, instead of collecting the raw data as in centralized machine learning, FL algorithms, such as federated averaging (FedAvg) \cite{fl}, preserve data privacy by allowing clients to independently update their local models, which are then aggregated by a trusted server through multiple training rounds.

Depending on the types of clients, FL systems can be classified into the cross-device and cross-silo settings \cite{federated_survey_2}.
In this work, we focus on the cross-device FL, where clients are typically edge devices (e.g., mobile phones and healthcare gadgets) with limited data and computational resources.
In cross-device FL systems, a grand challenge is the vast heterogeneity of data distributions among devices, namely the non-independent and identically distributed (non-IID) issue.
The existence of non-IID data is attributed by various factors such as different user behaviors, uneven data collection, or varying environmental conditions.
The negative impacts of the non-IID issue in FL have been well explored through theoretical analysis \cite{fl_noniid, scaffold,federated_survey_2} and empirical experiments \cite{fl_experiments}.
With a distinct data distribution gap between clients, their local updates adversely become diverged after local training, thereby resulting in a misguided global model and accuracy degradation.


\begin{figure}
    \centering
    \includegraphics[width = 0.48\textwidth]{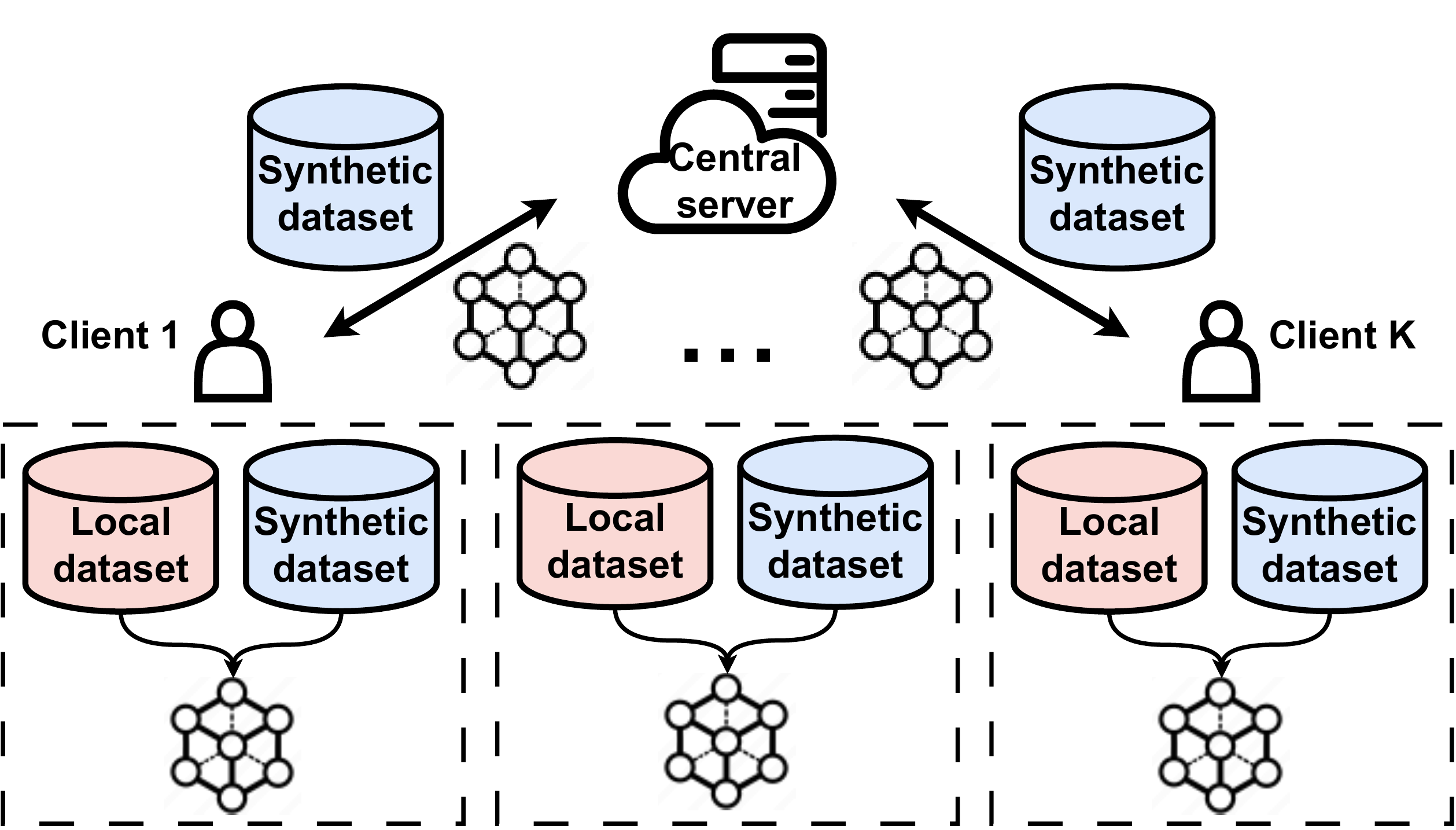}
    \caption{Overview of federated learning with data augmentation. Clients generate and share synthetic data with the central server for the global synthetic dataset constitution. Then the global synthetic dataset is sent back to clients for assisting the local training.}
    \label{fig:syn_framework}
\end{figure}

To mitigate data heterogeneity in cross-device FL, many methods propose to modify the local training or global aggregation phase under the FedAvg framework \cite{ma2022state}.
By leveraging the knowledge of the global model during the local training, the bias of local updates can be alleviated and thus the global model can converge along an unbiased direction \cite{fedprox, moon,FedDyn}.
Besides, simply averaging the biased local updates at the server may lead to an imprecise global model.
Therefore, adjusting the aggregation weights for different local models or finetuning the global model via the knowledge of local models \cite{fednova, fedgen,fedftg} can also relieve the non-IID problem.
These methods, however, focus on the algorithmic aspect and only achieve marginal performance gains, especially in the extreme non-IID cases where each client just has a few classes of data \cite{fl_experiments}. 

To fundamentally solve the non-IID issue, another stream of study mitigates the data distribution gap between clients via synthetic data augmentation \cite{fl_noniid,ma2022state}. 
Fig. \ref{fig:syn_framework} shows a general FL framework with synthetic data augmentation, where clients generate and share synthetic data with each other to relieve the data heterogeneity among clients.
Following this framework, Zhang et al. \cite{fedmix} utilized mixup data \cite{mixup} as the synthetic data to assist local training.
By averaging multiple samples and their corresponding labels, the mixup data become blurry and better preserve privacy compared with raw data.
However, since the mapping from the input space to the label space is non-linear, naively averaging images and labels leads to incorrect data pairs \cite{manifold_mixup}.
For example, the mixup data of digit `6' and digit `9' looks like digit `8', but with a label of fifty-fifty probabilities for `6' and `9', respectively.
Consequently, the data mixup method only achieves limited improvement in model performance. 
Besides, the mixup data contains rich information of real data in the input space, which exposes data privacy.

Alternatively, generative models such as generative adversarial networks (GANs) \cite{gan} and variational autoencoders (VAEs) \cite{vae} can be also used to generate synthetic data. 
Despite the incredible ability of generative models in centralized training, they suffer from the local overfitting problem in FL and require an extremely long training time when being trained with non-IID data. 
These challenges make it hard to be deployed in cross-device FL systems with limited data and scarce computational resources.
In addition, these generative models are optimized by learning the real data distribution, which promotes the synthetic data to be similar to real data, thus leading to severe privacy leakage.

This paper is dedicated to exploring the means of solving the non-IID issue in FL.
Specifically, we propose a novel data synthesis method based on model inversion techniques \cite{model_inversion_1}, which are able to overcome the aforementioned obstacles of existing data synthesis methods for FL.
By integrating the proposed data synthesis method with FL, we reduce the data heterogeneity among clients and improve model accuracy.
The main contributions of this paper are summarized as follows:

\begin{itemize}
    \item We propose a class-relevant feature matching data synthesis (\textit{FMDS}) method, which initializes the synthetic data with Gaussian noise and optimizes them by matching their class-relevant features with those of real data.
    This training paradigm in the input space instead of the model space significantly reduces the training overhead for data synthesis.
    By learning the class-relevant features indicated by the class activation map (CAM) and discarding the redundant features, FMDS can generate effective synthetic data for addressing the non-IID issue while preserving the privacy of raw data.
    \item { To further strengthen privacy protection and improve the effectiveness of synthetic data, we propose a hard feature-matching data synthesis (\textit{HFMDS}) method by promoting the real feature towards the decision boundary of different classes.
    The synthetic samples optimized by matching with the hard features can lead to a more precise decision boundary and thus improve the accuracy performance.
    The hard feature augmentation also eliminates more information about real features, thus further reducing privacy leakage.}
    \item By integrating the proposed HFMDS method with FL, we propose a novel FL algorithm with synthetic data augmentation (\textit{HFMDS-FL}), where each client can generate the hard-feature-matching synthetic data and share them with other clients, thus mitigating the data heterogeneity among clients.
    \item We present the visualization results and theoretical analysis to illustrate the effectiveness of HFMDS-FL in feature alignment and domain adaptation, resulting in solving the non-IID issue.
    We further verify this via simulations on benchmark datasets, and it is observed that our proposed HFMDS-FL framework consistently outperforms baselines regarding accuracy, privacy preservation, and computational cost saving.
\end{itemize}

\textbf{Organizations:} The rest of this paper is organized as follows.
In Section \ref{related works}, we review the related works on the non-IID problem and synthetic data augmentation in FL.
Section \ref{problem settings} describes the non-IID problem in FL and introduces a data synthesis-based solution and its corresponding challenges.
We propose a hard feature matching data synthesis (HFMDS) method and integrate it with FL in Section \ref{method}.
We analyze the effectiveness of our proposed algorithm in Section \ref{analysis} and then evaluate it via extensive simulations in Section \ref{experiments}.
Finally, we conclude the paper in Section \ref{conclusions}.

\section{Related Works}
\label{related works}
\textbf{Non-IID problem in FL:} 
The non-IID challenge of FL was first presented by \cite{fl}, which was later shown to significantly affect the convergence and performance of the global model \cite{fl_noniid,non_iid_convergence}. 
To overcome this issue, many works, known as the client-centric methods, attempted to reformulate the local training objective by leveraging knowledge of the global model and local models from other clients \cite{feature_alignment}.
FedProx \cite{fedprox} included a proximal term by restraining the local update using the global model.
SCAFFOLD \cite{scaffold} introduced the control variates to correct for the local drift during local training.
FedDyn \cite{FedDyn} proposed a dynamic regularizer for each client to parallelize the gradients among clients.
MOON \cite{moon} utilized contrastive learning to reduce the distance between model representations to correct the local training.
These methods, however, cannot solve the essence of the non-IID problem and may encounter the performance bottleneck in extreme cases with highly skewed data distributions \cite{fl_experiments}.

{Besides the local updates at clients, the server can also help to alleviate the negative impacts of non-IID data by calibrating the biased global model after aggregation. 
In CCVR \cite{classifier_calibration}, the server rectifies the classifier with virtual representations sampled from an approximated Gaussian mixture model.
FedFTG \cite{fedftg} calibrated the global model using data-free knowledge distillation with the knowledge of local models.}
Moreover, client clustering \cite{cluster_1,cluster_3} and client selection \cite{client_selection,client_selection_2} can also be conducted by the server to relieve the non-IID problem.
IFCA \cite{cluster_1} alternately estimated the cluster identities of clients via local empirical loss and updated the model parameters for each client cluster through gradient descent.
Wang et al. \cite{client_selection_2} proposed an algorithm that the server adaptively selects a subset of clients to maximize a reward that encourages the increase of validation accuracy under non-IID scenarios.
Our proposed algorithm is orthogonal to these methods and provides a new perspective to tackle the non-IID issue in FL.

\textbf{Data augmentation for non-IID challenge:} 
Recently, data augmentation techniques, including data manipulation and DL-based data synthesis, have been extensively investigated for generating synthetic data to be shared among clients in FL \cite{what2share}.
For data manipulation, a straightforward method is to allow clients to share a small portion of real data \cite{fl_noniid,data_share_2}, which, however, is contradictory to
the privacy-preserving principle of FL.
To relieve such a risk, some early attempts proposed to share the mixup data \cite{mixup}, which is obtained by averaging the real samples, to mitigate the data heterogeneity in FL systems.
Mix2FLD \cite{mix2fld} allowed clients to share mixup samples with the server and adopted an inverse engineer to generate inversely-mixup samples for federated distillation.
XorMixFL \cite{xormixup} let the server collect the processed data encoded using bit-wise XOR mixup and decode them for the centralized training. 
In FedMix \cite{fedmix}, clients share the mixup data with each other and update the local models using second-order optimization.
However, the unqualified mixup data caused by the mismatching between data and labels only achieve a limited performance improvement for FL, and the manipulation in the input space still severely leaks the privacy of real data.

The DL-based data synthesis methods conduct data synthesis by updating the generative models or optimizing the synthetic data directly.
A straightforward method is to train a qualified GAN \cite{gan} that can be shared among clients to generate synthetic data and features  \cite{fedgan,feddpgan,sda_fl,fedgen}.
FedGAN \cite{fedgan} and FedDPGAN \cite{feddpgan} allowed clients to collaboratively train a global generator under the FL framework and used it to generate synthetic data for mitigating data heterogeneity among clients.
In SDA-FL \cite{sda_fl}, clients individually train a local generator and share them with each other.
Although the GAN-based methods are promising in solving the non-IID problem, the high computational resource demands for training the generator and discriminator limit their applicability in practical systems.
Besides, VHL \cite{VHL} instituted a virtual homogeneous dataset to calibrate the features from the heterogeneous clients, which can be generated from pure noise shared across clients.
Moreover, zero-shot learning was also applied for synthetic data generation to promote the fairness of FL \cite{zero_shot_fl}, but only knowing the global model is incapable of generating synthetic data with sufficient quality.

This paper follows the DL-based data synthesis research line and proposes to learn the class-relevant features of real data, which can better solve the non-IID issue and preserve data privacy compared with the existing data augmentation methods in FL. 

\section{Problem Settings}
\label{problem settings}
\subsection{Non-IID Problem in FL}
We consider a typical FL setting for the general problem of $Y$-class classification. 
In the FL setting with $K$ clients, each client $k\in[K]$ has its local dataset $\hat{\mathcal{D}}_k$. The goal of FL is to train a global model that minimizes the global population loss as follows:
\begin{align}
    \mathcal{L}_g (\bm{w})= \mathbb{E}_{(\bm{x},y)\in \hat{\mathcal{D}}_g} \ell (f(\bm{w};\bm{x}), {y}),
\end{align}
 where $f: \mathcal{X} \rightarrow \mathcal{Y} $ is the function that maps the input data $\bm{x}$ to a predicted label $\hat{y}$ using the model $\bm{w}$, $\ell$ is a generic loss function, e.g., the cross-entropy loss, and ${\hat{\mathcal{D}}_g}=\cup_{k\in[K]} \hat{\mathcal{D}}_k $ is the global dataset of all the clients.
To achieve this goal, the global objective function for FL is given by:
\begin{align}
    \mathcal{L} (\bm{w}) = \frac{1}{K} \sum _{k=1 }^K \mathcal{L}_k(\bm{w}),
\end{align}
where $\mathcal{L}_k (\bm{w}) = \mathbb{E}_{(\bm{x},y)\in \hat{\mathcal{D}}_k} \ell ( f(\bm{w};\bm{x}), y)$ is the local objective function of client $k$.

To optimize the model $\bm{w}$, the classical FedAvg \cite{fl} allows clients to collaboratively train the global model in $T$ rounds.
In each communication round $t \in [T]$, the server randomly selects a subset of clients $\mathcal{K}_t$ from $K$ clients and sends the global model $\bm{w}_t$ to them.
As it is impractical to compute full gradients on the whole local dataset, each client $k \in \mathcal{K}_t$ updates the model locally with $\mathcal{T}_k$ local steps using the mini-batch SGD algorithm with the local dataset $\hat{\mathcal{D}}_k$.
Specifically, in each local step $\tau = 1, \cdots, \mathcal{T}_k$ of the $t$-th communication round, client $k$ randomly samples a batch of training data $\xi_{k}^{t,\tau}$ and updates local model $\bm{w}_{k}^{t,\tau-1}$ with the gradients $\nabla \ell (\bm{w}_k^{t, {\tau-1}}; \xi_k^{t,\tau})$ as follows:
\begin{align}
    \bm{w}_{k}^{t,\tau} \leftarrow \bm{w}_{k}^{t,\tau-1} - \eta\nabla \ell (\bm{w}_k^{t, {\tau-1}}; \xi_k^{t,\tau}),
\end{align}
where $\eta$ is the learning rate.
After the local training, each active client $k \in \mathcal{K}_t$ uploads its local model $\bm{w}_k^t$ to the server, all of which are aggregated to generate a new global model as follows:
\begin{align}
    \bm{w}^{t+1} = \frac{1}{|\mathcal{K}_t|} \sum_{k \in \mathcal{K}_t} \bm{w}^{t}_{k}, \quad \bm{w^}t_k \leftarrow \bm{w}_{k}^{t,\mathcal{T}_k}.
\end{align}
Afterwards, the new global model $\bm{w}^{t+1}$ is broadcast to the active clients for the next-round training.

If the data is IID in FL, the local gradient $\nabla \ell (\bm{w}_{k}^{t,\tau-1},\xi_{k}^{t,\tau} )$ at each client $k$ is close to the gradients of global population loss $\nabla \mathcal{L}_g (\bm{w})$.
The global model after aggregation thus achieves similar performance as centralized training.
However, when data is non-IID, for each client $k$, due to the distinct data distribution, the divergence between local gradients $\nabla \ell (\bm{w}_{k}^{t,\tau-1},\xi_{k}^{t,\tau} )$ and $\nabla \mathcal{L}_g (\bm{w})$ becomes larger and thereby leads to a biased global model after aggregation.
The bias of the global model is accumulated during the training process and thus degrades the convergence speed and model accuracy \cite{fl_noniid}.

\subsection{Data Synthesis for Non-IID Problem}
To bridge the data distribution gap between clients and reach an IID data distribution, one ideal solution is to share real samples among clients.
With the assistant of the real dataset $\hat{\mathcal{D}}_{-k}$ from other clients, the local objective function of client $k$ is formulated as follows:
\begin{align}
\label{ideal}
    &\mathcal{L}^*_k (\bm{w}_k ) \triangleq \alpha \mathbb{E}_{(\bm{x},y) \in \hat{\mathcal{D}}_k} \ell \left( f(\bm{w}_k;\bm{x}), y \right) \nonumber \\
    & \quad\quad\quad\quad\quad\quad + (1- \alpha) \mathbb{E}_{(\bm{x},y) \in \hat{\mathcal{D}}_{-k}} \ell \left( f(\bm{w}_k;\bm{x}), y \right),
\end{align}
where $\alpha$ is a hyperparameter weighting the loss terms.
With a proper value of $\alpha$, this local objective function enables the closer gradients to the gradients of global population loss $\nabla \mathcal{L}_g (\bm{w})$ and thus solve the non-IID issue.
However, this solution that requires real data sharing directly reveals the data privacy of clients.
To relieve the privacy leakage and solve the non-IID issue, instead of sharing the real data, a general FL framework is to share the synthetic data among clients.
As shown in Fig. \ref{fig:syn_framework}, each client generates its own synthetic dataset and shares it to the server for the global synthetic dataset constitution, which is shared back to clients for local regularization.
With the shared synthetic dataset $\hat{\mathcal{D}}_{\text{syn}}$, the new local objective function that approximates \eqref{ideal} is given by:
\begin{align}
    &\hat{\mathcal{L}}^*_k (\bm{w}_k ) \triangleq \alpha \mathbb{E}_{(\bm{x},y) \in \hat{\mathcal{D}}_k} \ell \left( f(\bm{w}_k;\bm{x}), y \right)  \nonumber \\
    & \quad\quad\quad\quad\quad\quad+ (1- \alpha) \mathbb{E}_{(\hat{\bm{x}},y) \in \hat{\mathcal{D}}_{\text{syn}}} \ell \left( f(\bm{w}_k;\hat{\bm{x}}), y \right).
    \label{local update}
\end{align}

However, there are three key challenges of effective data synthesis in FL detailed as follows:
\begin{itemize}
    \item \textbf{Data quality:} The synthetic data quality significantly affects the convergence rate and accuracy for FL.
    Only the synthetic data carrying the task-relevant features can replace the real data to reduce the data heterogeneity among clients, and the invalid synthetic data contributes negatively to the training process and impairs the accuracy.
    For example, the mixup data used in FedMix \cite{fedmix} achieve limited performance improvement.
    This is because the vanilla averaging in the input space and label space results in the mismatching data pairs \cite{manifold_mixup}.

    \item \textbf{Privacy leakage:} 
    To preserve privacy, the synthetic data should carry minimum information about the real images, i.e., the synthetic data should be dissimilar to the real data.
    However, since the mixup data used in FedMix are processed directly through the real data and the GAN-based synthetic data are generated by imitating the real data distribution, they look realistic and similar to the real data, thereby revealing the privacy of real data. 

    \item \textbf{Computational cost:} 
    Considering the limited computational resources at edge devices, the computation demands for data synthesis should be kept sufficiently low to ensure feasibility in cross-device FL systems.
    However, the GAN-based synthetic methods require extremely high computation costs for training the generators and discriminators, which is not compatible with the cross-device systems.
\end{itemize}
To verify the above discussion, we conduct a case study of FedMix \cite{fedmix}, FedGAN \cite{fedgan}, and our proposed method HFMDS-FL on CIFAR-10 dataset \cite{cifar10} in a cross-device FL system with 20 clients, and each client has 2 classes of data.
The results in Fig. \ref{fig:acc_psnr_cost} show that FedMix obtains an excessively high Peak Signal-to-Noise Ratio (PSNR) value and FedGAN requires extremely high computational costs.
In comparison, HFMDS-FL is able to overcome these challenges and achieves better performance in terms of accuracy, privacy preservation, and computational cost saving.
In the next section, we will introduce the HFMDS, which is then integrated with FL to develop HFMDS-FL.
\begin{figure}
    \centering
    \includegraphics[width = 0.4\textwidth]{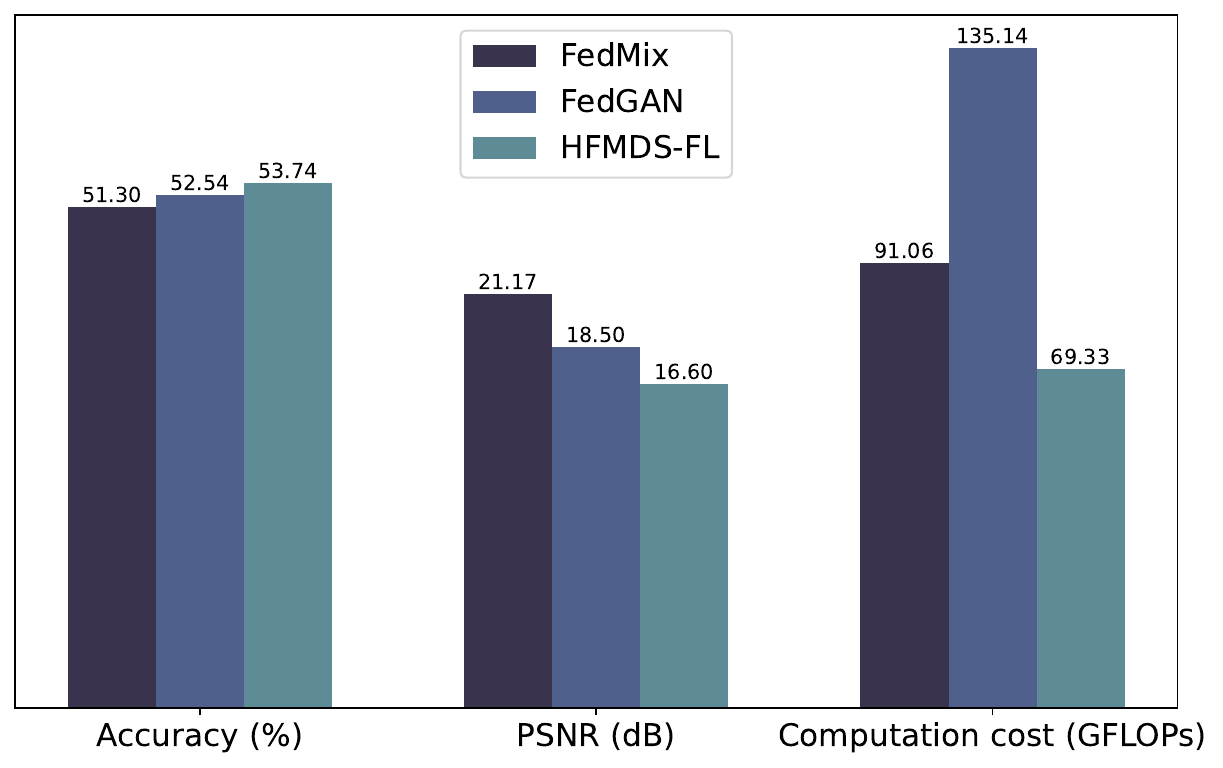}
    \caption{Comparison of FedMix, FedGAN, and our proposed algorithms (FMDS-FL and HFMDS-FL) in terms of accuracy, PSNR, and computational cost on CIFAR-10.}
    \label{fig:acc_psnr_cost}
\end{figure}

\section{Hard Feature Matching Data Synthesis for Federated Learning (HFMDS-FL)}
\label{method}
To overcome the above challenges, in this section, we first propose a novel data synthesis method using class-relevant feature matching based on model inversion techniques, and a hard feature augmentation mechanism is then proposed to provide stronger privacy protection while improving the utility for the synthetic data.
The overview of our proposed {HFMDS} method is shown in Fig. \ref{fig:framework}.
By incorporating the proposed HFMDS, we present HFMDS-FL  in Section \ref{HFMDS-FL}, which is a new FL framework with synthetic data augmentation to solve the non-IID issue.

\subsection{Model Inversion Data Synthesis}
To reduce the computational cost in data synthesis, we adopt the model inversion techniques \cite{model_inversion_1,model_inversion_2}, which optimizes the synthetic samples in the input space via the classification knowledge of real samples.
Compared with training high-dimensional GANs \cite{feddpgan,fedgan}, the synthetic samples are low-dimensional and easier to be trained via the inversion-based methods, thus saving valuable computational costs at clients.
The model inversion optimization problem for the synthetic data $\hat{\bm{x}}$ is formulated as:
\begin{equation}
\label{model_inversion}
     \min_{\hat{\bm{x}}} \mathcal{L}_C + \mathcal{R} (\hat{\bm{x}}),
\end{equation}
where $\mathcal{L}_C \triangleq \ell \left( f(\bm{w};\hat{\bm{x}}), y \right)$ is the classification loss function with the corresponding target label $y$, and $\mathcal{R} (\hat{\bm{x}})$ is the image regularization term for data synthesis.
With the initialization of the synthetic sample $\hat{\bm{x}}$ using Gaussian noise and an arbitrary target label $y$, we can update the synthetic sample $\hat{\bm{x}}$ by using the gradient descent to minimize \eqref{model_inversion}.
As such, the synthetic sample can learn abundant knowledge from the model and raw data.

Existing works on data generation \cite{model_inversion_3,dreaming,imagine,always_dream,batch_gradients} and model inversion attacks \cite{gradient_inversion, idlg} design the regularization terms $\mathcal{R} (\hat{\bm{x}})$ by including the total variance and $\ell_2$ norm of $\hat{\bm{x}}$ to make synthetic samples more realistic. Different from these works, we design an objective function to promote synthetic samples to maintain only the essential features that are beneficial to the classification task and discard redundant features, which aims at achieving better privacy protection while ensuring satisfactory classification accuracy. 

\begin{figure*}
    \centering
    \includegraphics[width=\textwidth]{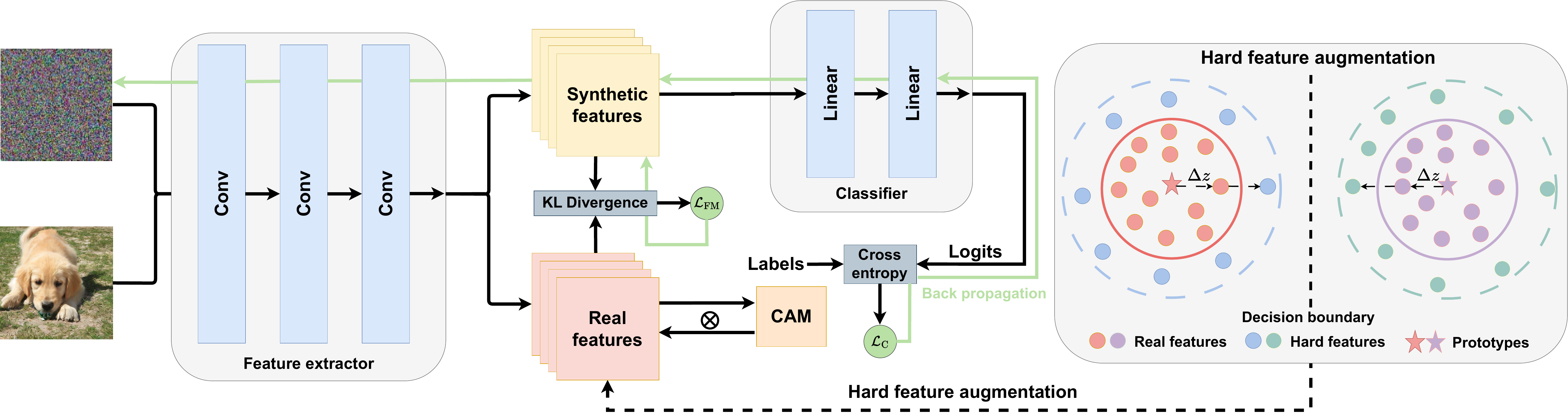}
    \caption{The overview of feature matching data synthesis.
    For FMDS, we compute the CAM for the real features to indicate the class-relevant features and then promote the synthetic features to match with them.
    While for HFMDS, instead of computing the CAM for the real features, we augment the real features by transferring them towards the decision boundary and then compute the CAM for the hard features.}
    \label{fig:framework}
\end{figure*}

\subsection{Class-relevant Feature Matching Data Synthesis (FMDS)}
\label{FMDS}


During the training process of FL, the global model becomes more informative to the classification task and is able to extract the discriminative features as the communication round increases. 
Specifically, the global model $\bm{w}$, which consists of a feature extractor $\bm{w}^e$ and a classifier $\bm{w}^d$, extracts the feature of the real sample $\bm{x}$ according to $\bm{z}=f^e (\bm{w}^e;\bm{x})$, where $f^e: \mathcal{X} \rightarrow \mathcal{Z}$ denotes the feature mapping function.
However, the extracted features $\bm{z}$ contain both class-specific and redundant features, and learning all of them makes the synthetic sample similar to the real sample, arising privacy concerns.
Therefore, we encourage synthetic samples to only learn the useful class-specific features.

To achieve this goal, we adopt the class activation map (CAM) \cite{cam,grad_cam} by utilizing the gradients of the corresponding class with respect to the feature maps to localize the class-specific features.
With the output logits $\bm{q}=f^d(\bm{w}^d;\bm{z})$ from the classifier $\bm{w}^d$ by inputting the extracted features $\bm{z}$, the CAM is defined as the gradients of output logits $\bm{q}^y$ of class $y$ with respect to features $\bm{z}$ at as follows:
\begin{align}
    \bm{g}_z = \frac{\partial {\bm{q}^y}}{\partial{\bm{z}}}.
\end{align}
Note that higher positive values of gradients $\bm{g}_z$ represent higher importance of feature maps for the corresponding class $y$ and vice versa.
The class-specific features can be indicated by the positive value of gradients and we can obtain the positive CAM gradients via the rectified linear unit (ReLU) function: $\text{ReLU} (\bm{g}_z)$.
To encourage the synthetic samples to preserve these class-specific features, we propose a one-on-one class-relevant feature matching loss function for optimizing the synthetic samples. 
By pushing the synthetic sample $\hat{\bm{x}}$ to match only the class-specific features with the real sample $\bm{x}$, we optimize the synthetic samples by minimizing the Kullback–Leibler (KL) divergence $D_{\text{KL}}$ between them as follows:
\begin{align}
\label{loss_fm}
    &\mathcal{L}_{\text{FM}} \triangleq D_{\text{KL}} \left( \hat{\bm{z}} \cdot \text{ReLU} (\bm{g}_z) || \bm{z} \cdot \text{ReLU} (\bm{g}_z) \right) \nonumber \\
    &= D_{\text{KL}} \left( f^e (\bm{w}^e; \hat{\bm{x}})  \cdot \text{ReLU} (\bm{g}_z) || f^e (\bm{w}^e; \bm{x})  \cdot \text{ReLU} (\bm{g}_z) \right).
\end{align}
With this objective function, only those essential features that are related to its corresponding label $y$ are indicated by the CAM and then learned by the synthetic sample $\hat{\bm{x}}$. 
Compared with prior works that use generative networks to imitate the distributions of real samples, i.e., they learn all the features of real samples, the objective function $\mathcal{L}_{\text{FM}}$ enables synthetic samples to maintain only the partial useful information of real samples, which thereby preserves the privacy of real samples.
The class-relevant features also guarantee the accuracy for the classification task.

\subsection{Hard Feature Matching Data Synthesis (HFMDS)}
\label{AFMDS_}
As will be demonstrated in the next section, the synthetic data generated by the proposed FMDS method can be used for feature alignment across clients and solves the non-IID issue. 
However, the output synthetic data $\hat{\bm{x}}$ in \eqref{loss_fm} still carries some information of real sample $\bm{x}$ and thus exposes partial privacy.
To better preserve the privacy of real sample $\bm{x}$ and maintain the effectiveness of the synthetic sample $\hat{\bm{x}}$ for classification, instead of matching with the real feature $\bm{z}$, we encourage the synthetic sample $\hat{\bm{x}}$ to learn the semantic augmented features $\bm{z}+ \mu \Delta \bm{z}$, where $\Delta \bm{z}$ is the transformation direction and $\mu > 0$ is the scaling factor for semantic feature augmentation, respectively.
Inspired by the hard sample mining techniques \cite{hard_sample_1, hard_sample_2} that utilize hard samples to smooth the decision boundary, as shown in Fig. \ref{fig:framework},  we transfer the real feature along the hard transformation direction so that the semantic augmented feature $\bm{z} + \mu \Delta \bm{z}$ is closer to the decision boundary and the synthetic data $\hat{\bm{x}}$ learned from it becomes the hard sample.

To obtain the hard transformation direction $\Delta \bm{z}$ for semantic feature augmentation, we search for an easy transformation direction $-\Delta \bm{z}$ so that the hard transformation direction is determined.
Since the features of samples from the same class are clustered together during the training, we formulate the \textit{prototype} $\Bar{\bm{z}}_c$, which is the center of the features of the samples from class $c$, and then use it to estimate the easy transformation direction as $- \Delta \bm{z}= \Bar{\bm{z}}_c - \bm{z}$.
Consequently, we obtain the hard transformation direction $\Delta \bm{z}=   \bm{z} - \Bar{\bm{z}}_c$.
Specifically, to compute the prototype $ \Bar{\bm{z}}_c$ for class $c$ during the $t$-th training round, we accumulate an intermediate feature set of $N_c$ real samples of class $c$ as $\{ \bm{z}_{c,n}^t | n=1, \cdots, N_c \}$ and then compute the prototype $\bm{z}_c^t$ by averaging the feature set: $\bm{z}_c^t= \frac{1}{N} \sum_{n=1}^N \bm{z}_{c,n}^t$.
To prevent the prototype oscillation and improve the effectiveness of augmented features, we adopt the epoch momentum for  updating the prototype of each class $c\in [Y]$: $\bar{\bm{z}}_c^t = (1-\lambda) \bm{z}_c^t + \lambda \bm{z}_c^{t-1}$ with a momentum factor $\lambda \in [0,1]$. 
By applying the hard transformation direction $\Delta \bm{z} = \bar{\bm{z}}_c^t - \bm{z}$ and the scaling factor $\mu >0 $, the semantic hard features $\bm{z}_{{h,c}}$ for the real features $\bm{z}$ of class $c$ can be formulated as follows:
\begin{align}
    \bm{z}_{h,c} = \bm{z} + \mu \Delta \bm{z} = (1+\mu) \bm{z} - \mu \bar{\bm{z}}_c^t.
\end{align}
With the corresponding CAM $\bm{g}_{z_{h,c}}=\frac{\partial {\bm{q}^y}}{\partial{\bm{z}_{h,c}}}$ for semantic hard features $\bm{z}_{{h,c}}$, which can be obtained by inputting $\bm{z}_{h,c}$ into the classifier and then computing the gradients $\bm{g}_{z_{h,c}}$, the loss function $\mathcal{L}_{\text{HFM}}$ for hard feature matching data synthesis (HFMDS) is reformulated as follows:
\begin{align}
    &\mathcal{L}_{\text{HFM}} \triangleq D_{\text{KL}} \left( \hat{\bm{z}} \cdot \bm{g}_{z_{h,c}} || \bm{z}_{h,c} \cdot \bm{g}_{z_{h,c}} \right) \nonumber \\
    &= D_{\text{KL}} \left( f^e (\bm{w}^e; \hat{\bm{x}})  \cdot \bm{g}_{z_{h,c}} || \left[ (1+\mu) f^e (\bm{w}^e; \bm{x}) - \mu \bar{\bm{z}}_c^t \right]  \cdot \bm{g}_{z_{h,c}} \right).
\end{align}

Compared with the partial real features in \eqref{loss_fm}, the synthetic data learned with the semantic hard features $\bm{z}_{h,c}$ carry the features that are closer to the decision boundary, with which the trained model becomes more generalized. 
Meanwhile, the augmentation for the real features erases more features of real samples, which is controlled by the scaling factor $\mu$, thereby further preserving the privacy of real samples.
From the mixup perspective, compared with FedMix which exploits the vanilla mixup at the input space, our proposed semantic feature augmentation method promotes a manifold mixture between the real feature $\bm{z}$ and a virtual feature $(1+2\mu) \bm{z} - 2\mu \bar{\bm{z}}_c^t$ at the feature space: $\bm{z}_{h,c} = (1+\mu) \bm{z} - \mu \bar{\bm{z}}_c^t = \frac{1}{2} \left( \bm{z} +  (1+2\mu) \bm{z} - 2\mu \bar{\bm{z}}_c^t \right)$, with which the synthetic data can smooth the decision boundary for the model and thus guarantee their utility \cite{manifold_mixup}.

Combined with the aforementioned classification loss function $\mathcal{L}_C \triangleq \ell \left( f(\bm{w};\hat{\bm{x}}), y \right)$, the objective function of HFMDS for optimizing the synthetic sample $\hat{\bm{x}}$ with the corresponding real sample $\bm{x}$ is expressed as follows:
\begin{align}
    \mathcal{L}_{\text{HFMDS}} =  \mathcal{L}_{\text{HFM}} + \mathcal{L}_{{C}}. 
    \label{overall_loss}
\end{align}
Please note that the label $y$ of the synthetic data should be defined as that of the real sample $\bm{x}$ since we optimize the synthetic sample $\hat{\bm{x}}$ together with the one-on-one feature matching loss $\mathcal{L}_{\text{HFM}}$. 


\begin{algorithm}
    \renewcommand{\algorithmicrequire}{\textbf{Input:}}
    \renewcommand{\algorithmicensure}{\textbf{Output:}}
    \caption{HFMDS-FL}
    \label{algo}
    \begin{algorithmic}[1]
    \REQUIRE Communication round ${T}$; data synthesis duration $T_d$; local steps $\mathcal{T}_k$ for client $k$; training steps $\mathcal{T}_g$ for data synthesis; client number $K$; global model $\bm{w}$ (contains extractor $\bm{w}^e$ and classifier $\bm{w}^d$); learning rate $\eta$; batch size $B$; Synthetic dataset size $B_g$;
        \STATE \textbf{Initialization}: model $\bm{w}_0$
        \FOR{ each communication round $t=1,\dots, {T}$}
        \IF{$t \mod T_d = 0 $ and $t \ne 0$}
        \FOR{each client $k \in \mathcal{K}$ \textbf{in parallel}}
        \STATE $ \hat{\mathcal{D}}_{k,\text{syn}} \leftarrow \text{DataSynthesis}(\bm{w}^t, \hat{\mathcal{D}}_k)$
        \ENDFOR
        \STATE Collect the synthetic datasets and aggregate them to constitute the synthetic dataset $\hat{\mathcal{D}}_{\text{syn}}$.
        \STATE All the clients download the synthetic dataset $\hat{\mathcal{D}}_{\text{syn}}$.
        \ENDIF 
        \STATE \textbf{Server Executes:}
        \STATE Sample active client set $\mathcal{K}_t$ from $K$ clients.
        \FOR{each client $k \in \mathcal{K}_t$ \textbf{in parallel}}
        \STATE $\bm{w}_k^t \leftarrow$ \text{ClientUpdate} ($\bm{w}^t,\hat{\mathcal{D}}_k$)
        \ENDFOR
        \STATE $\bm{w}^{t+1} = \sum_{k \in \mathcal{S}_t} |\frac{1}{\mathcal{K}_t}| \bm{w}_k^t$ \quad $\triangleright$ Global aggregation
        \ENDFOR
        \end{algorithmic}
        \textbf{def} \text{ClientUpdate($\bm{w}^t, \hat{\mathcal{D}}_k, \hat{\mathcal{D}}_{\text{syn}}$):}
        \begin{algorithmic}[1]
        \FOR{each step $\tau=1, \dots, \mathcal{T}_k$}
        \STATE $\bm{w}_{k}^{t,\tau} \leftarrow \bm{w}_{k}^{t,\tau-1} - \eta \nabla \ell (\bm{w}_{k}^{t,\tau-1},\xi_{k}^{t,\tau} )$ \quad $\triangleright$ Update local models via Equation \eqref{local update}
        \STATE Accumulate the real features for each class.
        \ENDFOR
        \FOR{each class $c = 1, \dots, Y$}
        \STATE Update the local prototype $\bar{\bm{z}}_c^t$.
        \ENDFOR
        \RETURN{$\bm{w}_{k}^{t,\tau}$}
        \end{algorithmic}
        \textbf{def} \text{DataSynthesis($\bm{w}_{t}, \hat{\mathcal{D}}_{k}$):}
        \begin{algorithmic}[1]
        \STATE Sample a real data subset $\hat{\mathcal{D}}_{k,{r}}$ from $\hat{\mathcal{D}}_k$ with size $B_g$ for dataset synthesis. 
        \STATE Initialize a synthetic dataset $\hat{\mathcal{D}}_{k, \text{syn}}$ with size $B_g$ by following a Gaussian distribution $\mathcal{N}(0,\bm{I})$.
        \FOR{each step $\tau = 1, \dots, \mathcal{T}_g$}
        \STATE update the synthetic dataset $\hat{\mathcal{D}}_{k, \text{syn}}$ using the loss function \eqref{overall_loss}.
        \ENDFOR
        \RETURN{$\hat{\mathcal{D}}_{k, \text{syn}}$}
        \end{algorithmic}
\end{algorithm}

\subsection{Hard Feature Matching Data Synthesis for Federated Learning (HFMDS-FL)}
\label{HFMDS-FL}
By integrating the proposed HFMDS method with FL, we propose the new framework of \textit{HFMDS-FL} to solve the non-IID issue in FL.
Different from existing works such as FedMix and FedGAN which conduct the data synthesis before the FL process, we perform data synthesis on-the-fly along the FL process.
This is because the quality of synthetic samples relies on the task relevance of real features, which can be better extracted by a well-trained global model.
Our proposed algorithm \textit{HFMDS-FL} is implemented based on FedAvg and the detailed procedures are summarized in Algorithm \ref{algo}.

The clients conduct data synthesis every $T_d$ communication rounds to reduce the computation burden.
In the $t$-th communication round, if the clients need to generate synthetic data, they first download the global model $\bm{w}_t$.
Each client $k\in[K]$ then samples a subset of data $\hat{\mathcal{D}}_{k,r}$ with $B_g$ training samples from $\hat{\mathcal{D}}_k$ and initializes a synthetic dataset $\hat{\mathcal{D}}_{k,\text{syn}}$ with $B_g$ training samples generated from the Gaussian distribution $\mathcal{N}(0,\bm{I})$.
The client optimizes the synthetic dataset $\hat{\mathcal{D}}_{k,\text{syn}}$ by minimizing the loss function in \eqref{overall_loss}.
After data synthesis, all the local synthetic datasets are collected by the server to constitute the global synthetic dataset $\hat{\mathcal{D}}_{\text{syn}}$, which is shared with clients for regularizing the local training.
In each local training step, each active client $k \in \mathcal{K}_t$ respectively samples $B$ real samples from $\hat{\mathcal{D}}_k$ and $B$ synthetic samples from $\hat{\mathcal{D}}_{\text{syn}}$, and update the local model using the objective function in \eqref{local update}.
During local training, the features of real samples extracted by the local model $\bm{w}^t_k$ are accumulated class-wise, which are used to update the local prototypes for each class.
After local training, the local models are uploaded to the server for global aggregation.

\section{Analysis of HFMDS-FL}
\label{analysis}
In this section, we provide multiple perspectives to understand HFMDS-FL. 
We first visualize what knowledge the synthetic data learn from the real data via HFMDS, then analyze why sharing the synthetic data generated by HFMDS is able to solve the non-IID issue, from the viewpoints of \textit{distribution matching} and \textit{domain adaptation}. 

\begin{figure*}
    \centering
    \includegraphics[width = \textwidth]{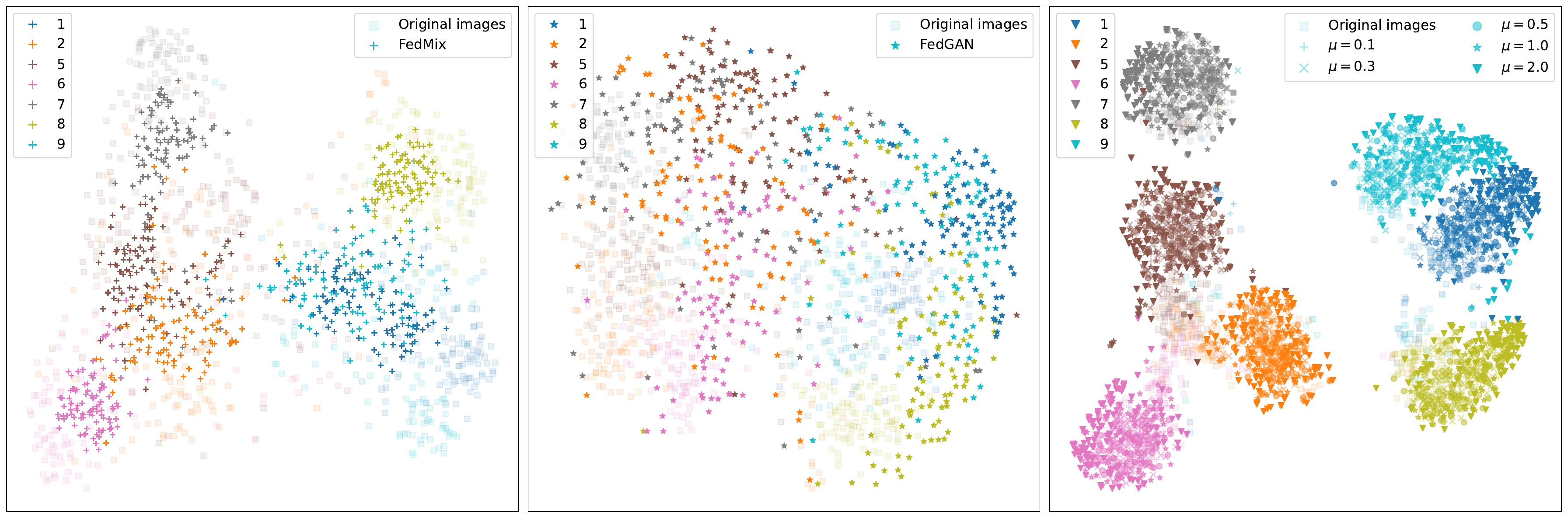}
    \caption{T-SNE plots of feature distribution for randomly selected classes of synthetic data generated by FedMix (left), FedGAN (middle), and HFMDS-FL (right) on CIFAR-10. The \textit{dark} dots represent synthetic data and the \textit{light} dots stand for real data.
    For HFMDS-FL, the synthetic samples generated with different values of $\mu$ are highlighted in varying levels of brightness, and the synthetic samples with larger values of $\mu$ are closer to the decision boundary.}
    \label{fig:tsne}
\end{figure*}


\subsection{Effectiveness of HFMDS}
\textbf{Feature Matching between Real and Synthetic Data:}
We sample the real data and generate the synthetic data using FedMix, FedGAN, and HFMDS-FL on CIFAR-10.
We visualize the feature distribution of their synthetic data in Fig. \ref{fig:tsne}.
With the mixture of the same class of data for FedMix, some classes of data cannot have a similar feature distribution of original real data and tangle with other classes of data, e.g., class 1 ({blue}) and class 9 ({cyan}). 
In FedGAN, the features of synthetic data generated by GANs have a distinct distribution gap with those of real data.
This is because the generative models require an extremely long time and numerous real data for training, and insufficient training time leads to unqualified synthetic data.
Compared with FedMix and FedGAN, our proposed method HFMDS-FL is able to generate synthetic data that have a compact and similar feature distribution with the real samples.
Also, with our proposed feature-matching objective function, the synthetic data quality of HFMDS-FL is insensitive to the real data number, generating the same-distribution synthetic data even with 100 real samples, which makes it more compatible with cross-device FL systems.

\textbf{Effectiveness of Hard Feature Augmentation:}
To analyze the effectiveness of our proposed HFMDS method, we visualize the feature distribution of synthetic data with varying scaling factors $\mu$ and highlight them with different levels of brightness.
The results in Fig. \ref{fig:tsne} show that with larger values of $\mu$, the features of synthetic samples move towards the decision boundary, which thus prevents model overfitting and improves the model's generalization capability.
In addition, with the larger values of $\mu$, the hard features (darker colors) are farther away from the original real features (the lightest color), which reduces privacy leakage compared with the original class-relevant FMDS method.
A sensitive analysis on scaling factors $\mu$ is conducted in Section \ref{ablation study} for further discussion.

\subsection{Feature Alignment across Clients}
A notable difference between HFMDS-FL and prior FL methods is that the shared synthetic data are utilized to regularize the local models.
Consequently, the synthetic data, which provide additional inductive bias for clients, reduce the local model bias by aligning features for each class across clients. To clarify, we present two following propositions.

\begin{proposition}
\label{remark1}
    With the local dataset $\hat{\mathcal{D}}_k$ at client $k$ and a real sample $(\bm{x}_{k}, y_{k})$ from it,
    by updating the local model with the real sample $(\bm{x}_{k}, y_{k})$, we can minimize the KL-divergence between two conditional feature distributions as follows:
    \begin{align}
        & \quad \min_{\bm{w}^e_k, \bm{w}^d_k} - \mathbb{E}_{(\bm{x}_k,y_k) \sim \hat{\mathcal{D}}_k}  \mathbb{E}_{ \bm{z} \sim p(\bm{z}|\bm{x}_{k}; \bm{w}^e_k)} \log p(y_k|\bm{z};\bm{w}^d_k) \nonumber \\
        &\equiv \min_{\bm{w}^e_k,\bm{w}^c_k} D_{\text{KL}} \left( p(z|\bm{x}_{k}; \bm{w}^e_k) || p(\bm{z}|{y}_{k}; \bm{w}^d_k) \right),
    \label{remark1_equa}
    \end{align}
    where we define $p(\bm{z}|{y}_{k}; \bm{w}^d_k)$ as the probability that the intermediate feature inputting into the classifier $\bm{w}_k^d$ is $\bm{z}$ given that it yields a label ${y}_{k}$.
    \begin{proof}
        Please refer to Appendix \ref{proof_remark_1}.
    \end{proof}
\end{proposition}

\begin{proposition}
\label{remark2}
    Let $(\bm{x}_{k^\prime}, y_{k^\prime})$ be the real sample from the local dataset $\hat{\mathcal{D}}_{k^\prime}$ of client $k^\prime$, and $(\hat{\bm{x}}_{k^\prime}, {y}_{k^\prime})$ is its corresponding synthetic sample.
    After the data synthesis, the synthetic dataset $\hat{\mathcal{D}}_{k^\prime, \text{syn}}$ of client $k^\prime$ is shared and used for local updating at client $k$.
    By regulating the local model of client $k$ with the synthetic sample $(\hat{\bm{x}}_{k^\prime}, y_{k^\prime})$, we can minimize the conditional representation KL-divergence between client $k$ and client $k^\prime$ as follows:
    \begin{align}
       & \min_{\bm{w}^e_k, \bm{w}^d_k} -  \mathbb{E}_{(\hat{\bm{x}}_{k^\prime},y_{k^\prime}) \sim \hat{\mathcal{D}}_{k^\prime, \text{syn}} } \mathbb{E}_{\bm{z} \sim p(\bm{z}|\hat{\bm{x}}_{k^\prime}; \bm{w}^e_k)} \log p(y_{k^\prime}|\bm{z};\bm{w}^d_k) \nonumber \\ \approx &\min_{\bm{w}^e_k,\bm{w}^d_k} D_{\text{KL}} \left( p(\bm{z}|\bm{x}_{k^\prime}; \bm{w}^e_g) || p(\bm{z}|{y}_{k^\prime}; \bm{w}^d_k) \right) \label{remark2_equa},
    \end{align}
    where $p(\bm{z}|{y}_{k^\prime}; \bm{w}^d_k)$ is the probability that the intermediate feature inputting into the classifier $\bm{w}_k^d$ is $\bm{z}$ given that it yields a label ${y}_{k^\prime}$, and $\bm{w}_g^e$ is the extractor of the global model used for data synthesis, respectively.
    \begin{proof}
        Please refer to Appendix \ref{proof_remark_2}.
    \end{proof}
\end{proposition}
Note the loss functions on the left-hand sides (LHSs) of \eqref{remark1_equa} and \eqref{remark2_equa} are the components of the local update of client $k$.
By updating the local models with the local objective function \eqref{local update}, we can minimize the KL-divergence terms in the right-hand sides (RHSs) of \eqref{remark1_equa} and \eqref{remark2_equa}.
Therefore, the KL divergence between the representation distributions of client $k$ (i.e., $p(\bm{z}|\bm{x}_{k}^i; \bm{w}^e_k)$) and client $k^\prime$ (i.e.,  $p(\bm{z}|\bm{x}_{k^\prime}^i; \bm{w}^e_g)$) can be reduced with the same label ${y}_{k}^i={y}_{k^\prime}^i$.
In other words, with the participation of synthetic data optimized by our proposed class-relevant feature-matching loss function, we can promote the feature alignment between the local model $\bm{w}^e_k$ of client $k$ and the global model $\bm{w}^e_g$ for the same class of data.
Consequently, the features for each class can be aligned across clients and thus the non-IID problem is solved.
For HFMDS-FL, the real sample $x_{k^\prime}^i$ can be regarded as a virtual hard sample, and the synthetic sample $\hat{\bm{x}}_{k^\prime}^i$ is optimized by matching the class-relevant features with it, which still achieves feature alignment across clients.

\subsection{Generalization Error Bound Analysis}
To further demonstrate the effectiveness of our proposed HFMDS-FL, we establish an upper bound for the generalization error of federated learning.
To ensure clarity, following \cite{shao2023selective}, we supplement some relevant definitions before delving into the analysis.

\begin{definition}
    (\textit{Hypothesis ${\bm{h}}$})
    Let $\mathcal{V}$ be a one-hot vector set that contains $Y$ one-hot vectors for a $Y$-class classification task.
    Given a sample $\bm{x} \in \mathcal{X}$, a hypothesis $\bm{h}: \mathcal{X} \rightarrow \mathcal{V}$ outputs the one-hot vector $\bm{v} \in \mathcal{V}$ with $\bm{v}_c=1$ and $\bm{v}_j=0, j \neq c$ to represent the class label $c \in [Y]$. 
\end{definition}

\begin{definition}
    (\textit{Minimum combined loss}) The ideal predictor in the hypothesis space $\hat{\mathcal{H}}_k$ achieves the minimum combined loss over any two datasets. We define two representatives as follows:
    \begin{align}
        \lambda_k &= \min_{ \hat{\bm{h}}_k \in  \hat{\mathcal{H}}_k} \left\{ \mathcal{L}_{\hat{\mathcal{D}}_g} (\hat{\bm{h}}_k, \hat{\bm{h}}^*) + \mathcal{L}_{{\mathcal{D}}_{{k}}} (\hat{\bm{h}}_k, \hat{\bm{h}}^*) \right\},
    \end{align}
    and
    \begin{align}
        \lambda_{k,\text{syn}} &= \min_{ \hat{\bm{h}}_k \in  \hat{\mathcal{H}}_k} \left\{ \mathcal{L}_{\hat{\mathcal{D}}_g} (\hat{\bm{h}}_k, \hat{\bm{h}}^*) + \mathcal{L}_{{{\mathcal{D}}}_{\text{syn}}} (\hat{\bm{h}}_k, \hat{\bm{h}}^*) \right\},
    \end{align}
    where $\hat{\bm{h}}^*: \mathcal{X} \rightarrow \mathcal{V}$ is the labeling function to output the ground-truth class label, and we assume that global dataset and local datasets have the same labeling function.
\end{definition}
If the dataset ${\mathcal{D}}_k$ has a distinct distribution gap with the global dataset $\hat{\mathcal{D}}_g$, it is unlikely to find a local hypothesis $\hat{\bm{h}}_k$ to minimize the minimum combined loss.
As the synthetic dataset ${\mathcal{D}}_{\text{syn}}$ is collected from all the clients, it has a closer distribution with the global dataset $\hat{\mathcal{D}}_g$, which leads to a smaller minimum combined loss, i.e., $\lambda_{k, \text{syn}}< \lambda_{k}$.

\begin{definition}
    (\textit{Hypothesis space $\mathcal{G}_k$}) For a hypothesis space $\hat{\mathcal{H}}_k$, we define a hypothesis space $\mathcal{G}_k: \mathcal{X} \rightarrow \{0,1\}$ for the hypotheses $g_k$, and $g_k(\textbf{x}) = \frac{1}{2} \| \hat{\textbf{h}}_k (\textbf{x}) -  \hat{\textbf{h}}^\prime_k (\textbf{x})\|$ with $\hat{\textbf{h}}_k, \hat{\textbf{h}}^\prime_k \in \hat{\mathcal{H}}_k$. 
\end{definition}

\begin{definition}
    (\textit{$\mathcal{G}_k$-\text{distance}}) Given any two datasets $\mathcal{D}$ and $\mathcal{D}^\prime$ over $\mathcal{X}$, let $\mathcal{G}_k=\{ g_k: \mathcal{X} \rightarrow \{ 0,1\}\}$ be a hypothesis space, and the $\mathcal{G}_k$-\text{distance} between the distributions of $\mathcal{D}$ and $\mathcal{D}^\prime$ is defined as $d_{\mathcal{G}_k} (\mathcal{D}, \mathcal{D}^\prime)=2\sup_{g_k \in \mathcal{G}_k} | \Pr_{\mathcal{D}} [g_k(\textbf{x}=1)] -   \Pr_{\mathcal{D}^\prime} [g_k(\textbf{x}=1)] |$.
\end{definition}

With the above definitions, we provide the generalization error bound of the global hypothesis $\hat{h}$ on the global dataset $\hat{\mathcal{D}}_g$ in the following theorem.

\begin{theorem}
\label{theorem_convergence}
    Let $m_k$ and $m_{\text{syn}}$ be the size of the empirical local dataset $\hat{\mathcal{D}}_k$ and synthetic dataset $\hat{\mathcal{D}}_{\text{syn}}$ drawn from the $\mathcal{D}_k$ and $\mathcal{D}_{\text{syn}}$, respectively. 
    Denote $\hat{h}_k$ the hypothesis learned on the local dataset $\mathcal{D}_k$, and $\hat{h}=\frac{1}{K} \sum_{k=1}^K \hat{h}_k$ the global hypothesis ensembled by the local hypotheses. 
    Given any $\delta \in (0,1)$, the following holds:
    \begin{align}
        &\mathcal{L}_{\hat{\mathcal{D}}_{g}} (\hat{h}) \leq \frac{1}{K} \sum_{k=1}^K  \mathcal{L}_{\hat{\mathcal{D}}_{g}} (\hat{h}_k, \hat{h}^*) \nonumber \\ 
        &\leq \frac{1}{K} \sum_{k=1}^K \left[ \underbrace{\mathcal{L}_{\hat{\mathcal{D}}_k \cup \hat{\mathcal{D}}_{\text{syn}} } (\hat{h}_k)}_{\text{Local empirical loss}} + \underbrace{\sqrt{-\frac{1}{2}  (\frac{\alpha^2}{m_k} + \frac{(1-\alpha)^2}{m_{\text{syn}}} ) \log \frac{\delta}{2K}}}_{\text{Numerical constraint}} \right. \nonumber \\
        & \left.+ \alpha [ \lambda_k + d_{\mathcal{G}_k} ({\mathcal{D}}_k, \hat{\mathcal{D}}_g) ]
        + (1-\alpha) [\lambda_{k,{\text{syn}}} + d_{\mathcal{G}_k} ({\mathcal{D}}_{\text{syn}}, \hat{\mathcal{D}}_g) ] \right] \label{bound},
    \end{align}
     with probability at least $1-\delta$, where $\lambda_k$ is the minimum combined loss for the global dataset ${\hat{\mathcal{D}}}_g$ and local training dataset ${\mathcal{D}}_k$ of client $k$, and $\lambda_{k, \text{syn}}$ is the minimum combined loss for the global dataset ${\hat{\mathcal{D}}}_g$ and synthetic dataset ${\mathcal{D}}_{\text{syn}}$.
    \begin{proof}
        Please refer to Appendix \ref{proof_theorem1}.
    \end{proof}
\end{theorem}
In Theorem \ref{theorem_convergence}, the first term in the RHS of \eqref{bound} represents the empirical loss over the local training dataset $\hat{\mathcal{D}}_k$ and the proxy dataset $\hat{\mathcal{D}}_{syn}$, and the second term is a numerical constraint indicating that a larger size of proxy dataset benefits the generalization performance.
Note that when there is no synthetic dataset shared among clients, the hyperparameter $\alpha=1$ and the generalization bound in \eqref{bound} reduces to that of FedAvg.
The results in Theorem \ref{theorem_convergence} provides two key insights. Firstly, the generalization error bound is significantly affected by the $\mathcal{G}$-divergence between the local distribution ${\mathcal{D}}_k$ and the test distribution ${\hat{\mathcal{D}}}_g$, which is also the cause of the non-IID problem in FL.
Secondly, since our proposed synthetic data carry the class-relevant features as the real data and thus the synthetic dataset ${\mathcal{D}}_{\text{syn}}$ can reach a smaller data distribution gap with the global dataset $\hat{\mathcal{D}}_g$ compared to the local dataset ${\mathcal{D}}_k$.
Therefore, we have $\lambda_{k,\text{syn}} < \lambda_k$ and $d_{\mathcal{G}_k} ({\mathcal{D}}_{\text{syn}}, \hat{\mathcal{D}}_g) <  d_{\mathcal{G}_k} (\mathcal{D}_k, \hat{\mathcal{D}}_g)$.
Consequently, the presence of our proposed data synthetic data reduces the generalization error bound with $\alpha \in [0,1)$ compared with FedAvg, in which there is no synthetic data and $\alpha=1.0$.


\begin{table}[]
    \centering
    \caption{Experimental setup}
    \begin{tabular}{cc}
    \hline
     Parameters    & CIFAR-10/CIFAR-100  \\ \hline
      Batch size $B$   &  10  \\
      Synthetic dataset size $B_g$ of each client  &  100  \\
      Local epochs  & 1  \\
      Learning rate $\eta$ & 0.005  \\
      Momentum for local update  & 0.9  \\
      Weight decay  &$5 \times 10^{-4}$ \\
      Data synthesis steps $\mathcal{T}_g$ & 500 \\
      Data synthesis duration $T_d$  & 20 \\ 
      Hyperparameter $\alpha$ & 0.1  \\
      Scaling factor $\mu$   & 0.5 \\
      Momentum factor $\lambda$ & 0.5 \\
      \# of trials & 3\\
      \hline
    \end{tabular}
    \label{tab:setup}
\end{table}

\section{Performance evaluation}
\label{experiments}

In this section, we compare HFMDS-FL with baseline algorithms on the CIFAR-10 \cite{cifar10} and CIFAR-100 \cite{cifar10} datasets to demonstrate its advantages in solving the non-IID problem, preserving privacy, and saving computational costs. We also conduct ablation studies to investigate the effect of objective functions, hyperparameters, client numbers, and local epochs.

\textbf{Baselines:}
Apart from FedAvg \cite{fl}, we compare the proposed algorithms with benchmarking FL algorithms specialized for solving the non-IID problem, including FedProx \cite{fedprox}, MOON \cite{moon}, and FedGen \cite{fedgen}. Meanwhile, we also adopt two popular data augmentation algorithms for FL, including FedMix \cite{fedmix} and FedGAN \cite{fedgan} as baselines.
For fair comparisons, we do not include any differential privacy noise in the models of FedGAN.
The synthetic data in FedMix are averaged by two real samples.
To illustrate the effectiveness of our proposed hard feature augmentation method, we also simulate the performance of FMDS-FL without hard feature augmentation, whose objective function is $\mathcal{L}_{\text{FM}}+\mathcal{L}_C$.


\textbf{Datasets:}
We evaluate the algorithms on two benchmark classification datasets, i.e., CIFAR-10 \cite{cifar10} and CIFAR-100 \cite{cifar10}. There are ten classes of data and one hundred classes of data on CIFAR-10 and CIFAR-100, respectively, each of which comprises 50,000 training data samples and 10,000 testing data samples.
To simulate the non-IID scenarios, we adopt both label-skewed and Dirichlet data distribution among 20 clients.
For label-skewed distribution scenarios, the training data are randomly assigned to clients, and each of them has some classes of data.
Besides, we incorporate the Dirichlet distribution \cite{fl_experiments} to account for varying degrees of data heterogeneity, utilizing values of $\text{Dir}=0.01$ and $\text{Dir}=0.05$ to denote severe and mild heterogeneity, respectively.
The detailed experimental setup is summarized in Table \ref{tab:setup}.

\textbf{Model architecture:}
We adopt a convolutional neural network (CNN) model, which contains two $5 \times 5$ convolutional layers followed by $2 \times 2$ max pooling layers and two fully connected layers with the latent dimensions of $1,600$ and $512$, respectively.
We view the last fully connected layer as the classifier and the remaining layers as the feature extractor.

\textbf{Hyperparameters:}
For local training, we employ mini-batch SGD as the local optimizer, with a batch size of $10$, learning rate of $0.005$, momentum of $0.9$, and weight decay of $5\times 10^{-4}$.
In HFMDS-FL, we set the hyperparameter $\alpha=0.1$ for local training, $\mu=0.5$ for hard feature augmentation, and $\lambda=0.5$ for prototype update.
Each client generates $B_g=100$ synthetic data and optimizes them using an Adam optimizer with a learning rate $0.02$ in $\mathcal{T}_g=500$ steps.

\subsection{Performance Comparison}

\textbf{Accuracy performance:}
Table \ref{main_results} summarizes the test accuracy of FMDS-FL, HFMDS-FL, and baselines on CIFAR-10 and CIFAR-100 datasets.
Both FMDS-FL and HFMDS-FL outperform the baselines in all the non-IID scenarios, surpassing FedAvg by $7.52\%$ and $2.34\%$ on CIFAR-10 and CIFAR-100 when $\text{Dir}=0.01$, respectively, which demonstrates the effectiveness of synthetic data generated by our proposed FMDS and HFMDS methods.
FedProx regulates the local training using parameter alignment while it is inferior to FedAvg in these highly skewed non-IID scenarios.
To alleviate data heterogeneity, MOON utilizes contrastive learning between the global model and the local model, and FedGen adopts feature augmentation with a feature generator.
However, they achieve similar or slightly better accuracy than FedAvg in most non-IID settings.
Besides, FedMix suffers from an accuracy degradation in most non-IID scenarios compared with FedAvg due to the mismatched mapping between the mixup samples and mixup labels.
FedGAN outperforms FedAvg on CIFAR-10 but achieves an unsatisfactory performance on CIFAR-100.
This is because the generative models for CIFAR-100 are difficult to train and insufficient training leads to unqualified synthetic data.
Besides, compared with FMDS-FL, HFMDS-FL achieves better accuracy in all non-IID cases with $\mu=0.5$ on CIFAR-10 and CIFAR-100 datasets, which evidences that the augmented hard features are able to move towards and smooth the decision boundary during the training process.
We will analyze how $\mu$ affects the accuracy in ablation studies to illustrate the effectiveness of augmented hard features.

\begin{table*}[]
\centering
\caption{Accuracy ($\%$) of varying algorithms. Each experiment is repeated in three trials. The results \textbf{in bold} indicate the best performance and the second best results are \underline{underlined}.}
\begin{tabular}{ccccccc}
\cmidrule(lr){1-7}
{\multirow{2}{*}{Algorithms}} & \multicolumn{3}{c}{CIFAR-10}   & \multicolumn{3}{c}{CIFAR-100}                                                                     \\ \cmidrule{2-4} \cmidrule(lr){5-7}
                           & \multicolumn{1}{c}{Dir=0.01} & \multicolumn{1}{c}{Dir=0.05} & \multicolumn{1}{c}{2 class/client} & \multicolumn{1}{c}{Dir=0.01} & \multicolumn{1}{c}{Dir=0.05} & \multicolumn{1}{c}{10 class/client}  \\ \cmidrule(lr){1-7}
\multicolumn{1}{c}{FedAvg \cite{fl}}                      & \multicolumn{1}{c}{28.53 $\pm$ 6.74}    & \multicolumn{1}{c}{45.31 $\pm$ 1.64}   & \multicolumn{1}{c}{50.44 $\pm$ 2.70}          & \multicolumn{1}{c}{27.14 $\pm$ 0.46}    & \multicolumn{1}{c}{30.51 $\pm$ 0.24}   & \multicolumn{1}{c}{28.16 $\pm$ 0.24}         \\ 
\multicolumn{1}{c}{FedProx \cite{fedprox}}                     & \multicolumn{1}{c}{27.68 $\pm$ 5.31}    & \multicolumn{1}{c}{42.57 $\pm$ 2.10}   & \multicolumn{1}{c}{46.36 $\pm$ 2.80}          & \multicolumn{1}{c}{25.64 $\pm$ 0.62}    & \multicolumn{1}{c}{29.08 $\pm$ 0.32}   & \multicolumn{1}{c}{26.99 $\pm$ 0.42}        \\ 
\multicolumn{1}{c}{MOON \cite{moon}}                        & \multicolumn{1}{c}{28.20 $\pm$ 7.18}    & \multicolumn{1}{c}{44.81 $\pm$ 1.80}   & \multicolumn{1}{c}{50.24 $\pm$ 2.80}          & \multicolumn{1}{c}{27.20 $\pm$ 0.52}    & \multicolumn{1}{c}{30.67 $\pm$ 0.26}   & \multicolumn{1}{c}{28.10 $\pm$ 0.20}          \\ 
\multicolumn{1}{c}{FedGen \cite{fedgen}}                      & \multicolumn{1}{c}{28.33 $\pm$ 7.45}    & \multicolumn{1}{c}{45.66 $\pm$ 1.59}   & \multicolumn{1}{c}{51.33 $\pm$ 1.72}          & \multicolumn{1}{c}{26.80 $\pm$ 0.74}    & \multicolumn{1}{c}{30.48 $\pm$ 0.48}   & \multicolumn{1}{c}{28.78 $\pm$ 0.22}         \\ 
\multicolumn{1}{c}{FedMix \cite{fedmix}}                      & \multicolumn{1}{c}{26.27 $\pm$ 2.41}    & \multicolumn{1}{c}{45.11 $\pm$ 1.57}   & \multicolumn{1}{c}{51.30 $\pm$ 3.54}          & \multicolumn{1}{c}{26.80 $\pm$ 0.58}    & \multicolumn{1}{c}{30.30 $\pm$ 0.43}   & \multicolumn{1}{c}{28.15 $\pm$ 0.25}        \\ 
\multicolumn{1}{c}{FedGAN \cite{fedgan}}                      & \multicolumn{1}{c}{32.75 $\pm$ 5.41}    & \multicolumn{1}{c}{45.93 $\pm$ 1.14}   & \multicolumn{1}{c}{52.54 $\pm$ 0.49}          & \multicolumn{1}{c}{26.82 $\pm$ 0.74}    & \multicolumn{1}{c}{30.00 $\pm$ 0.07}   & \multicolumn{1}{c}{27.87 $\pm$ 0.13}        \\ 
\multicolumn{1}{c}{FMDS-FL}                     & \multicolumn{1}{c}{\underline{{33.67} $\pm$ {3.49}}}    & \multicolumn{1}{c}{\underline{{45.95} $\pm$ {0.99}}}   & \multicolumn{1}{c}{\underline{52.69 $\pm$ 0.58}}          & \multicolumn{1}{c}{\underline{28.82 $\pm$ 0.76}}    & \multicolumn{1}{c}{\underline{31.22 $\pm$ 0.39}}   & \multicolumn{1}{c}{\underline{{29.81} $\pm$ {0.19}}}       \\ 
\multicolumn{1}{c}{HFMDS-FL}                    & \multicolumn{1}{c}{$\bm{36.05} \pm \bm{3.08}$}    & \multicolumn{1}{c}{{$\bm{46.73} \pm \bm{1.70}$}}   & \multicolumn{1}{c}{$\bm{53.74}$ $\pm$ $\bm{0.26}$}          & \multicolumn{1}{c}{$\bm{29.48} \pm \bm{0.79}$}    & \multicolumn{1}{c}{$\bm{31.54} \pm \bm{0.32}$}   & \multicolumn{1}{c}{$\bm{29.92} \pm \bm{0.36}$}       \\ \cline{1-7}
\end{tabular}
\label{main_results}
\end{table*}

\begin{table*}[]
\centering
\caption{PSNR (dB) of synthetic data and training computational costs (GFLOPs) of varying algorithms. The results \textbf{in bold} indicate the best performance in privacy and computational costs and the second best results are \underline{underlined}.}
\begin{tabular}{ccccccccc}
\cmidrule(lr){1-9}
  \multicolumn{1}{c}{\multirow{3}{*}{Algorithms}}                          & \multicolumn{3}{c}{PSNR (dB)}                                                                                       & \multicolumn{1}{c}{\multirow{3}{*}{\makecell{Computational costs \\ (GFLOPs)}}} & \multicolumn{3}{c}{PSNR (dB)}                                                                           & \multicolumn{1}{c}{\multirow{3}{*}{\makecell{Computational costs \\(GFLOPs)}}}   \\ \cmidrule(lr){2-4} \cmidrule(lr){6-8}
 & \multicolumn{3}{c}{CIFAR-10}                                                                                   & \multicolumn{1}{c}{}                                              & \multicolumn{3}{c}{CIFAR-100}                                                                      & \multicolumn{1}{c}{}                                            \\ \cmidrule(lr){2-4} \cmidrule(lr){6-8}
\multicolumn{1}{c}{}                            & \multicolumn{1}{c}{Dir=0.01}       & \multicolumn{1}{c}{Dir=0.05} & \multicolumn{1}{c}{2 class/client}        & \multicolumn{1}{c}{}                                              & \multicolumn{1}{c}{Dir=0.01} & \multicolumn{1}{c}{Dir=0.05} & \multicolumn{1}{c}{10 class/client} & \multicolumn{1}{c}{}                                             \\ \cmidrule(lr){1-9}
\multicolumn{1}{c}{FedMix}                      & \multicolumn{1}{c}{21.13}          & \multicolumn{1}{c}{21.86}          & \multicolumn{1}{c}{21.17}          & \multicolumn{1}{c}{91.06}                                         & \multicolumn{1}{c}{15.78}    & \multicolumn{1}{c}{15.84}   & \multicolumn{1}{c}{15.74}           & \multicolumn{1}{c}{91.56}                                      \\ 
\multicolumn{1}{c}{FedGAN}                      & \multicolumn{1}{c}{17.89}          & \multicolumn{1}{c}{18.48}          & \multicolumn{1}{c}{18.50}          & \multicolumn{1}{c}{135.14}                                        & \multicolumn{1}{c}{17.30}    & \multicolumn{1}{c}{17.49}   & \multicolumn{1}{c}{17.39}           & \multicolumn{1}{c}{135.38}                                       \\ 
\multicolumn{1}{c}{FMDS-FL}                     & \multicolumn{1}{c}{\underline{16.98}}          & \multicolumn{1}{c}{\underline{17.32}}          & \multicolumn{1}{c}{\underline{16.69}}          & \multicolumn{1}{c}{$\bm{69.32}$}                                         & \multicolumn{1}{c}{\underline{12.41}}    & \multicolumn{1}{c}{\underline{10.96}}   & \multicolumn{1}{c}{\underline{13.38}}           & \multicolumn{1}{c}{$\bm{69.56}$}                               \\ 
\multicolumn{1}{c}{HFMDS-FL}                    & \multicolumn{1}{c}{$\bm{15.59}$} & \multicolumn{1}{c}{$\bm{16.53}$} & \multicolumn{1}{c}{$\bm{16.60}$} & \multicolumn{1}{c}{\underline{69.33}}                                         & \multicolumn{1}{c}{$\bm{12.36}$}    & \multicolumn{1}{c}{$\bm{10.33}$}   & \multicolumn{1}{c}{$\bm{13.15}$}           & \multicolumn{1}{c}{\underline{69.57}}                     \\ \cmidrule(lr){1-9}                             
\end{tabular}
\label{privacy_cost}
\end{table*}

\textbf{Privacy and computational cost performance:}
Following \cite{gradient_inversion}, we introduce the Peak Signal-to-Noise Ratio (PSNR) to quantitatively evaluate the privacy protection of the synthetic data, which is defined as $\text{PSNR}=10\log_{10}(\frac{\text{MAX}_\text{I}^2}{\text{MSE}})$, with $\text{MAX}_\text{I}$ as the maximum possible pixel value and $\text{MSE}$ as the mean square error between the synthetic sample and real sample. 
A larger PSNR value means higher similarity between the synthetic samples and the original samples, indicating more severe privacy leakage.

We calculate the average PSNR values of all the synthetic samples and the average computational costs for each client in each communication round for the data-augmented FL algorithms, including FedMix, FedGAN, FMDS-FL, and HFMDS-FL, and summarize the results in Table \ref{privacy_cost}. 
The synthetic data generated by our proposed FMDS achieves lower PSNR values than FedMix and FedGAN (e.g., with decreasing PSNR values of $4.88$ dB and $6.53$ dB compared with FedMix and FedGAN on CIFAR-100 with $\text{Dir}=0.05$, respectively), demonstrating that our proposed FMDS method provides better privacy protection for the local data.
In addition, our proposed HFMDS-FL further decreases the PSNR values and protects the privacy of real samples better.
This is because the real features are augmented towards a hard transformation direction and thus more real information is erased.

In addition, as shown in Table \ref{privacy_cost}, compared with FedMix and FedGAN, our proposed FMDS-FL and HFMDS-FL require lower computational costs for data synthesis and local training.
This is because, in FedMix, the second-order based regularization term used for data augmentation during local training significantly increases the computational costs.
Moreover, FedGAN requires an extremely long time to train the generators and discriminators locally, thus requiring the most computational costs among these data synthesis methods, nearly twice more than HFMDS-FL.
Instead of training a network, our proposed FMDS-FL and HFMDS-FL utilize a one-on-one feature-matching objective function to optimize the synthetic samples directly, thus substantially decreasing the training costs compared with the GAN-based data-augmented methods.

\textbf{Convergence performance:}
Fig. \ref{fig:convergence} shows the convergence performance of FMDS-FL, HFMDS-FL, and baselines on CIFAR-10 and CIFAR-100 datasets.
FMDS-FL and HFMDS-FL achieve significantly faster convergence rates compared with the baselines.
Although FedGAN converges faster than other baselines by incorporating the synthetic data at the beginning,
our proposed FMDS-FL and HFMDS-FL algorithms achieve faster convergence rates after the 20-th communication round.
This is because the data synthesis process is conducted every 20 communication rounds in these settings and thus the synthetic data are used for regulating the local training from the 20-th communication round, which indicates that the synthetic data generated by FMDS-FL and HFMDS-FL have remarkably higher quality than the synthetic data of other baselines and thus improve the performance when they are used for training.

\begin{figure*}
    \centering
    \includegraphics[width=0.9\textwidth]{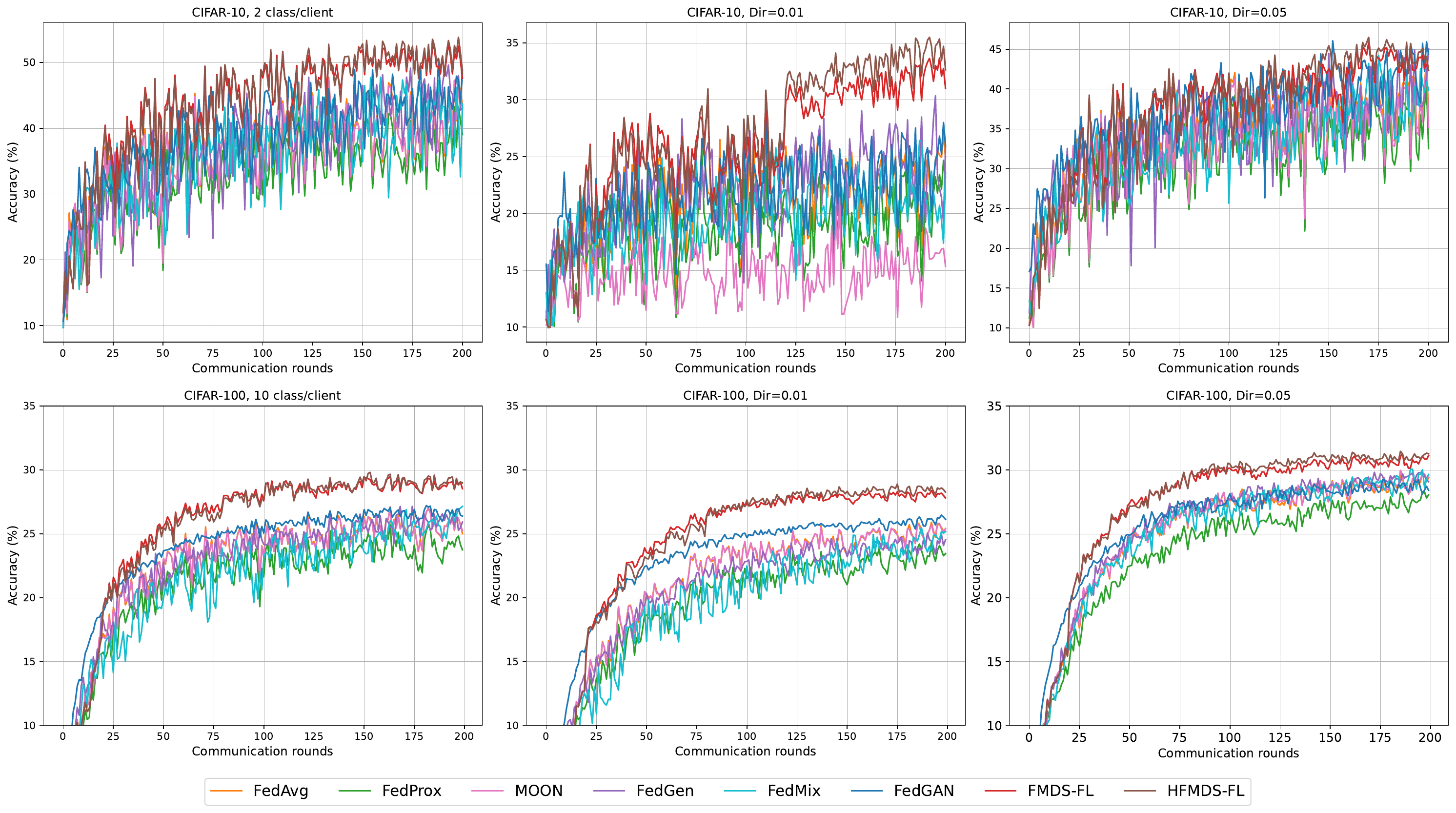}
    \caption{Convergence performance of varying algorithms on CIFAR-10 and CIFAR-100.}
    \label{fig:convergence}
\end{figure*}

\begin{table}[h]
    \centering
    \caption{Accuracy ($\%$) of FMDS-FL and HFMDS-FL with varying loss terms for data synthesis on CIFAR-10 and CIFAR-100 with 2 classes and 10 classes of data at each client, respectively. $\mathcal{L}_{\text{EFM}}$ represents the easy feature matching loss with easy feature transformation direction with $\mu=-0.5$, and $\mu=0.5$ is adopted for hard feature matching loss function $\mathcal{L}_{\text{HFM}}$}
    \begin{tabular}{ccc}
    \cmidrule(lr){1-3}
       Objective functions  & CIFAR-10  & CIFAR-100  \\ \cmidrule(lr){1-3}
        $\mathcal{L}_C$ & 52.08 & 27.04 \\ 
        $\mathcal{L}_C+\mathcal{L}_{\text{EFM}}$ & 52.58 & 28.03 \\ 
        $\mathcal{L}_C+\mathcal{L}_{\text{FM}}$ (FMDS-FL) & 52.86 & 29.50 \\ 
        $\mathcal{L}_C+\mathcal{L}_{\text{HFM}}$ (HFMDS-FL) & $\bm{53.82}$ & $\bm{29.73}$ \\ \cmidrule(lr){1-3}
    \end{tabular}
    \label{tab:hyperparameter}
\end{table}

\begin{figure}[!t]
    \centering
    \includegraphics[width=0.35 \textwidth]{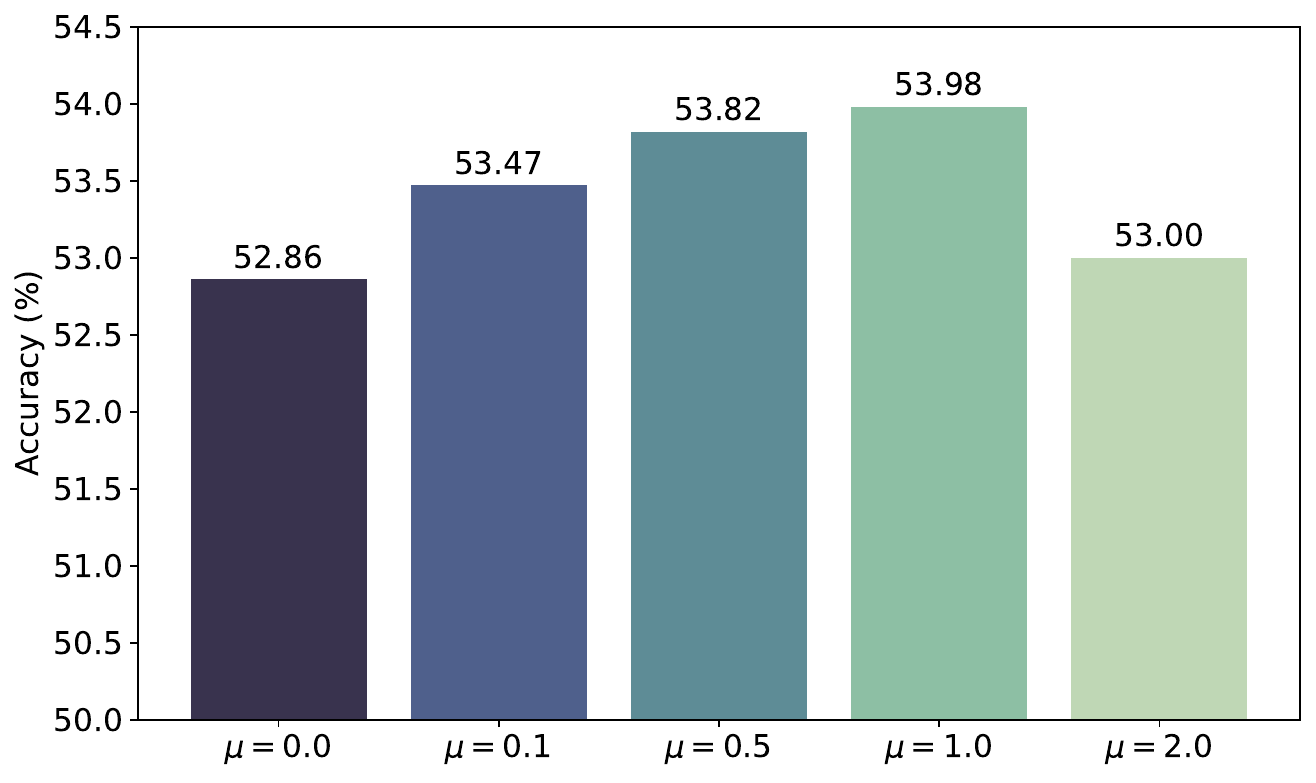}
    \caption{Accuracy ($\%$) of HFMDS-FL with varying values of $\mu$ on CIFAR-10 with 2 class of data at each client. $\mu=0.0$ stands for the FMDS-FL algorithm without the hard feature augmentation.}
    \label{fig:mu_cifar10}
\end{figure}

\begin{figure}[!t]
    \centering
    \includegraphics[width=0.35 \textwidth]{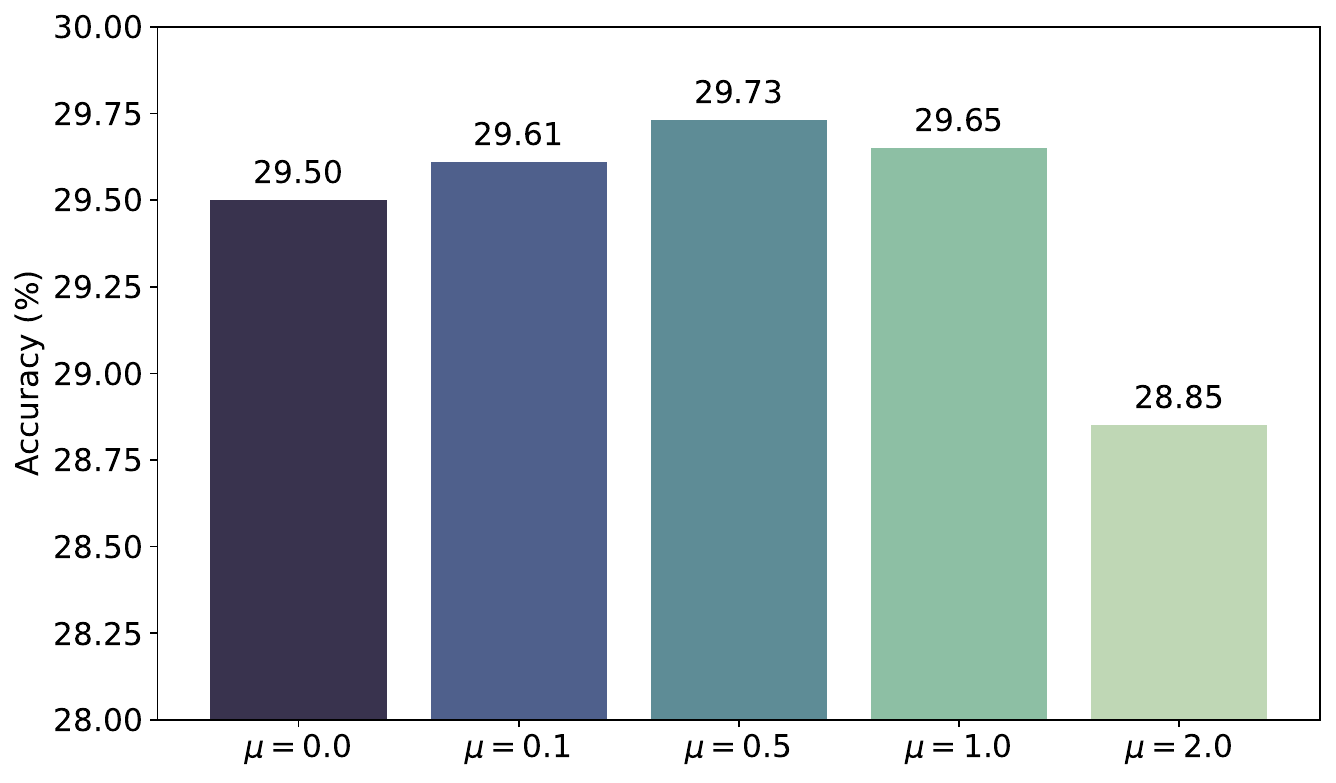}
    \caption{Accuracy ($\%$) of HFMDS-FL with varying values of $\mu$ on CIFAR-100 with 10 classes of data at each client. $\mu=0.0$ stands for the FMDS-FL algorithm without the hard feature augmentation.}
    \label{fig:mu_cifar100}
\end{figure}

\subsection{Ablation Studies}
\label{ablation study}

\textbf{Effectiveness of the loss functions for data synthesis:} To illustrate the effectiveness of the proposed objective functions in \eqref{overall_loss} for data synthesis, we test the accuracy of FMDS-FL and HFMDS-FL with various objective functions on the CIFAR-10 and CIFAR-100 datasets.
To further demonstrate the effectiveness of hard feature augmentation, we evaluate the accuracy of HFMDS-FL with easy feature augmentation (i.e., $\mu=-0.5$), which transfers the reals features towards the prototypes.
As the pairwise comparison shown in Table \ref{tab:hyperparameter}, by updating the objective functions progressively, FMDS-FL outperforms the algorithm with only the classification loss $\mathcal{L}_C$, which demonstrates the effectiveness of our proposed feature matching data synthesis method in generating the high-quality synthetic samples for mitigating the non-IID issue in FL.
Besides, the easy feature augmentation decreases the accuracy compared to FMDS-FL while HFMDS-FL improves the accuracy on both datasets, which proves the benefits of our proposed hard feature augmentation method.


\textbf{Effectiveness of hard feature augmentation for data synthesis:}
We analyze how the value of $\mu$ affects the accuracy by setting various values of $\mu$ on CIFAR-10 and CIFAR-100 with two classes and ten classes of data in each client, respectively.
The results in Fig \ref{fig:mu_cifar10} and Fig \ref{fig:mu_cifar100} show that properly increasing the value of $\mu$ can improve the accuracy compared with FMDS-FL without the hard feature augmentation, which verifies that the hard feature towards decision boundary promotes the generalization of models and thereby improves the accuracy.
However, excessively increasing the values of $\mu$ leads to an accuracy decline (e.g., $0.65\%$ accuracy drops when increasing $\mu$ from 0 to 2.0). 
This is due to the fact that a large value of $\mu$ facilitates the synthetic samples to move across the decision boundary and entangle with other classes of data, thus generating unqualified synthetic samples.


\begin{figure}[!t]
    \centering  
    \includegraphics[width=0.35 \textwidth]{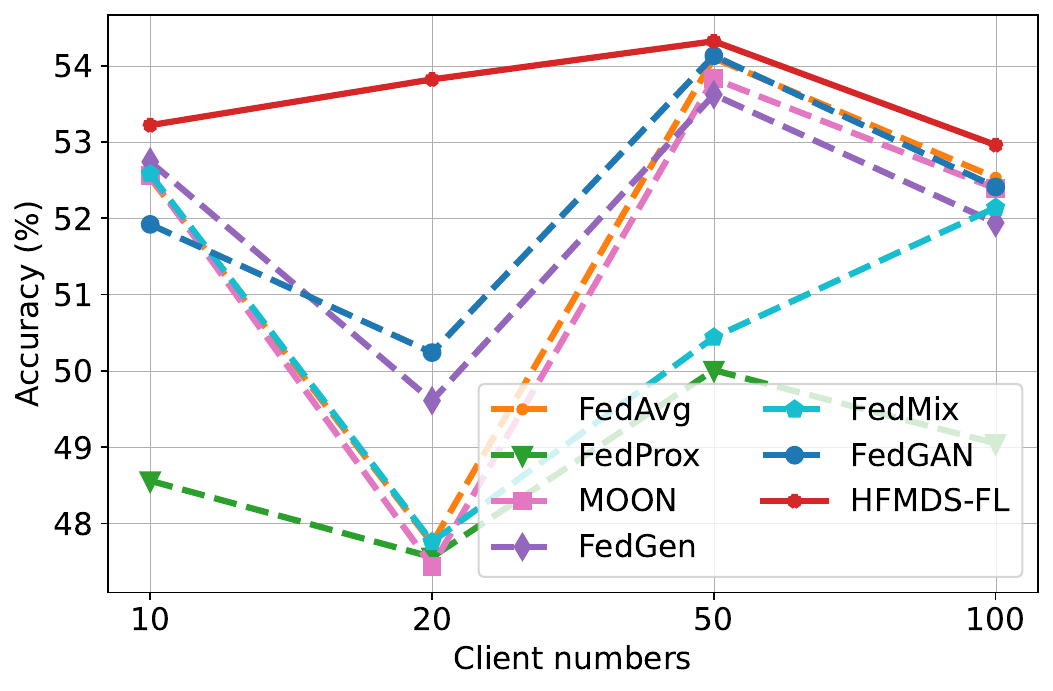}
    \caption{Accuracy ($\%$) of different algorithms with varying client numbers on CIFAR-10, and ten of them are randomly active in each communication round.}
    \label{fig:client_num_cifar10}
\end{figure}

\begin{figure}[!t]
    \centering
    \includegraphics[width=0.35 \textwidth]{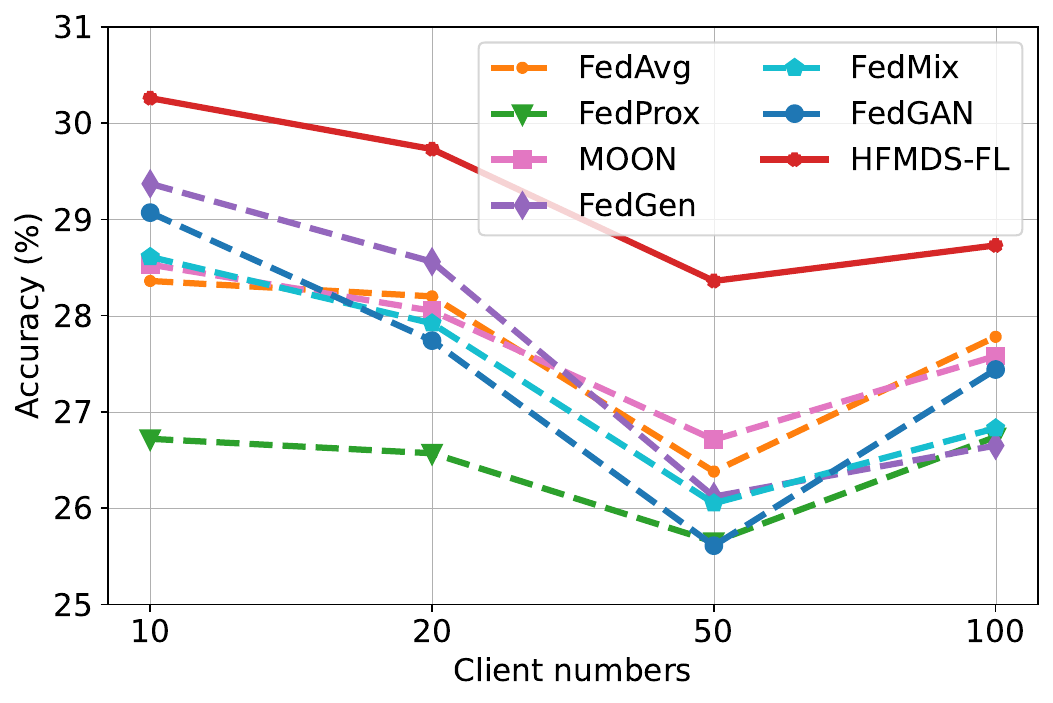}
    \caption{Accuracy ($\%$) of different algorithms with varying client numbers on CIFAR-100 and ten of them are randomly selected in each communication round.}
    \label{fig:client_num_cifar100}
\end{figure}

\textbf{Impacts of client numbers:}
We investigate the impact of client numbers on HFMDS-FL and the baselines on the CIFAR-10 and CIFAR-100 datasets.
Fig. \ref{fig:client_num_cifar10} and Fig. \ref{fig:client_num_cifar100} show that, with different client numbers, HFMDS-FL achieves the best performance in all the settings.
For example, HFMDS-FL outperforms FedGAN by 0.82$\%$ and 1.29$\%$ with 100 clients on the CIFAR-10 and CIFAR-100 datasets, respectively, which highlights the scalability and deployability of HFMDS-FL on real-world systems. 

\textbf{Impacts of local epoch numbers:}
We study how the number of local epochs affects the accuracy of all the methods on CIFAR-10 and CIFAR-100.
As depicted in Fig. \ref{fig:local_epoch_cifar10} and Fig. \ref{fig:local_epoch_cifar100}, HFMDS-FL is robust against all the baselines with different numbers of local epochs, particularly on the CIFAR-100 dataset, with at least 1.17$\%$ accuracy improvement over the second-best algorithm, demonstrating its reliability and generalization ability.

\section{Conclusions}
\label{conclusions}
In this paper, we addressed the data statistical heterogeneity challenge in FL by proposing a novel hard feature matching data synthesis (HFMDS) method based on model inversion techniques.
We optimized the synthetic data to be task-relevant and privacy-preserving by matching the class-relevant features of real data.
A hard feature augmentation mechanism that pushes the real features towards the decision boundary was further proposed to preserve data privacy while improving accuracy.
By integrating the proposed HFMDS with FL, we provided a novel FL framework named HFMDS-FL with data augmentation to solve the non-IID issue.
The theoretical analysis demonstrated that the generated synthetic data boost feature alignment across clients and lead to smaller generalization error bound in FL.
Furthermore, the extensive experiments showed that HFMDS-FL consistently outperforms the baselines in accuracy, privacy preservation, and computational costs.
For future work, it is worth investigating how to employ pre-trained generative models for data synthesis in FL to further improve the performance.

\begin{figure}[!t]
    \centering
    \includegraphics[width=0.35 \textwidth]{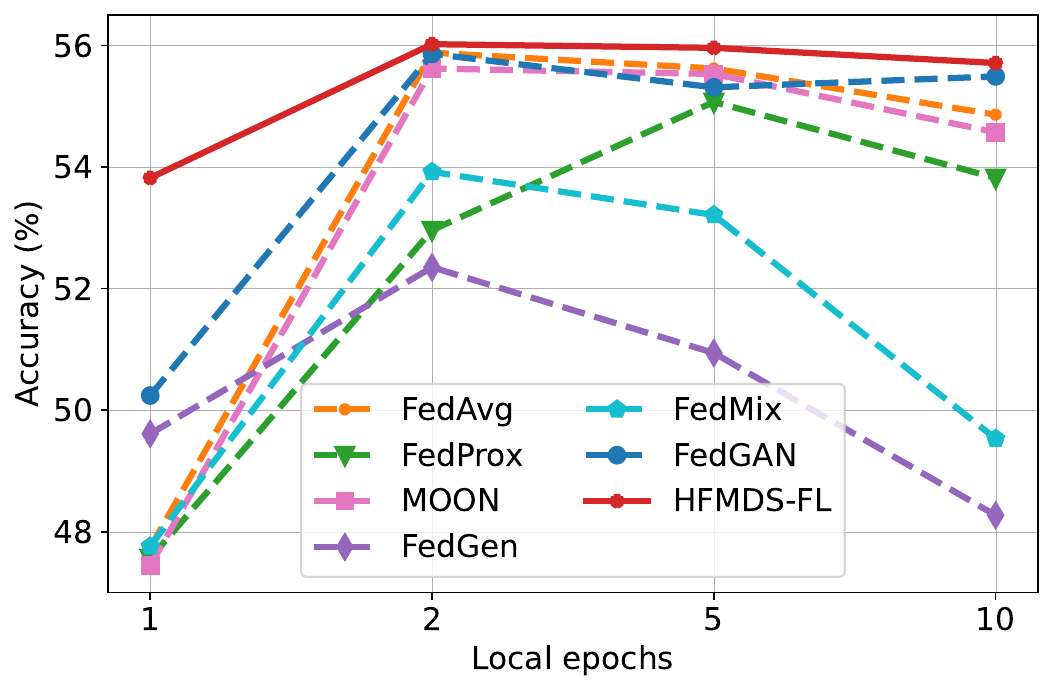}
    \caption{Accuracy ($\%$) of different algorithms with varying local epochs on CIFAR-10.}
    \label{fig:local_epoch_cifar10}
\end{figure}

\begin{figure}[!t]
    \centering
    \includegraphics[width=0.35 \textwidth]{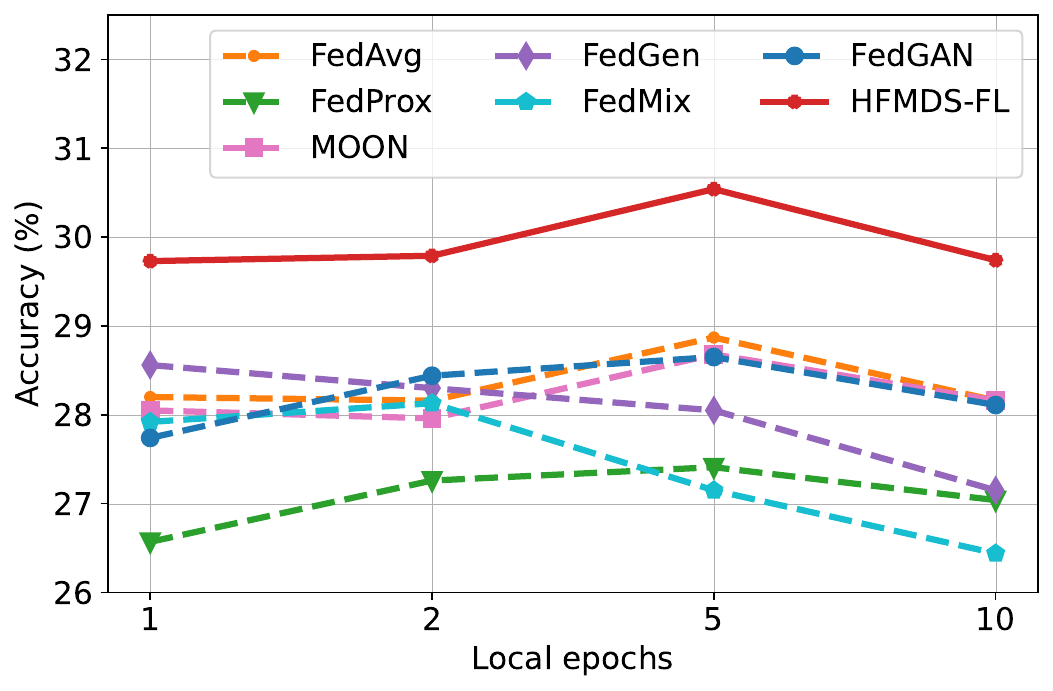}
    \caption{Accuracy ($\%$) of different algorithms with varying local epochs on CIFAR-100.}
    \label{fig:local_epoch_cifar100}
\end{figure}

\bibliographystyle{IEEEtran}
\bibliography{references}

\begin{IEEEbiography}[{\includegraphics[width=1in,height=1.25in,clip,keepaspectratio]{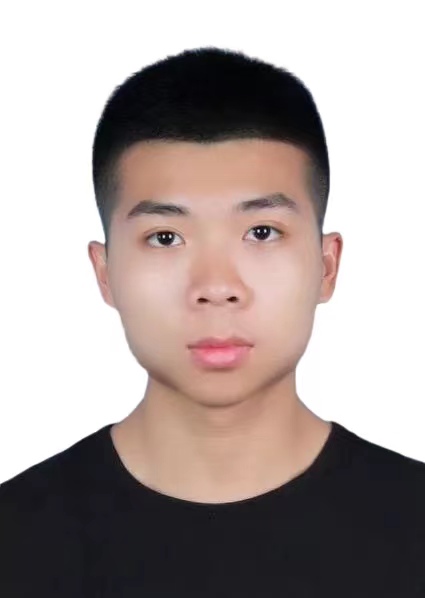}}]{Zijian Li}
(Graduate student member, IEEE) received the B.Eng. degree in Electrical Engineering and Automation from the South China University of Technology in 2020, and the M.Sc. degree in electronic and information engineering from the Hong Kong Polytechnic University in 2022.
He is currently pursuing a Ph.D. degree in the Department of Electronic and Computer Engineering at the Hong Kong University of Science and Technology.
His research interest is federated learning.
\end{IEEEbiography}
\vskip -2\baselineskip plus -1fil

\begin{IEEEbiography}[{\includegraphics[width=1in,height=1.25in,clip,keepaspectratio]{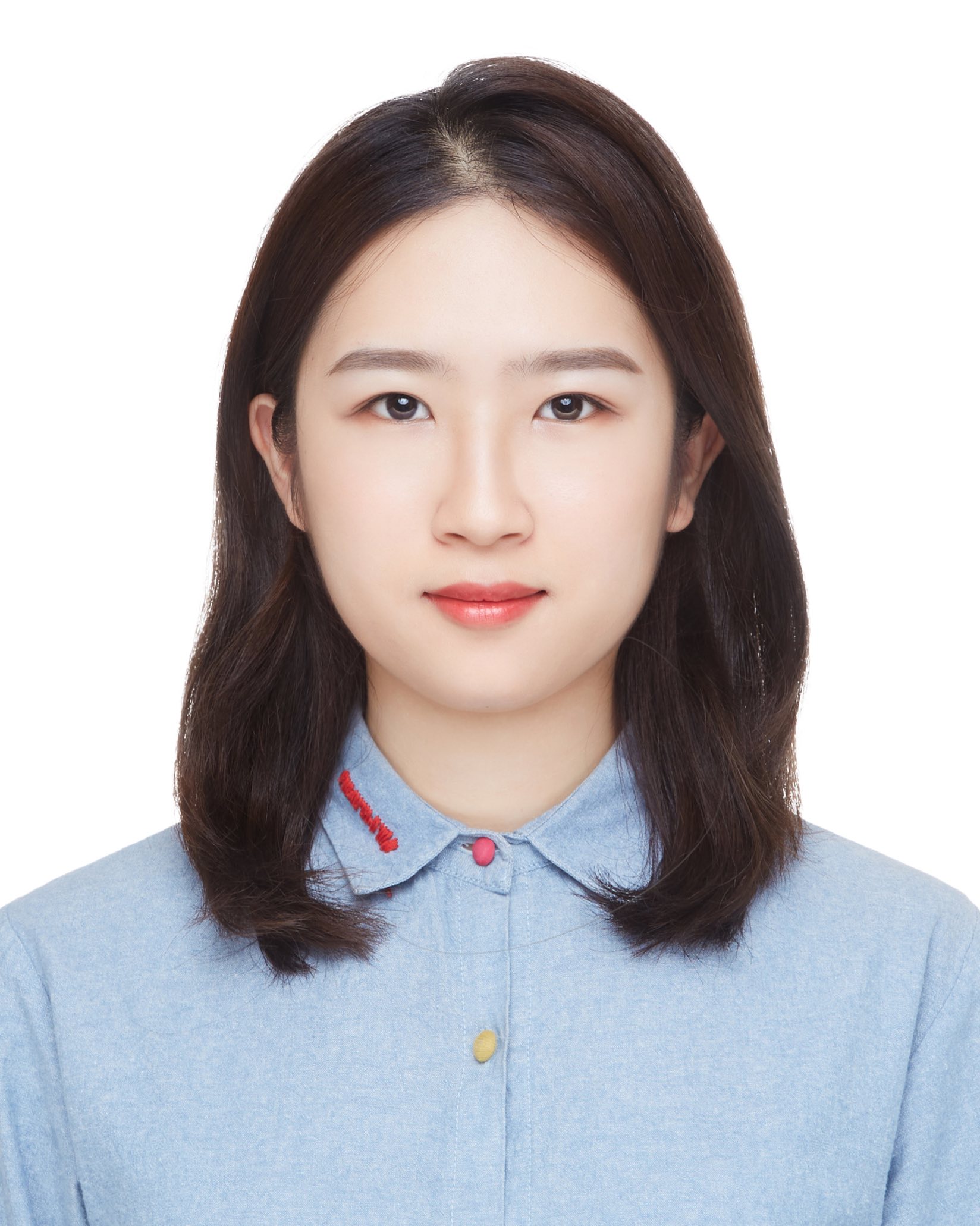}}]{Yuchang Sun}
(Graduate student member, IEEE) received the B.Eng. degree in electronic and information engineering from Beijing Institute of Technology in 2020. She is currently pursuing a Ph.D. degree at Hong Kong University of Science and Technology. Her research interests include federated learning and distributed optimization.
\end{IEEEbiography}
\vskip -2\baselineskip plus -1fil

\begin{IEEEbiography}[{\includegraphics[width=1in,height=1.25in,clip,keepaspectratio]{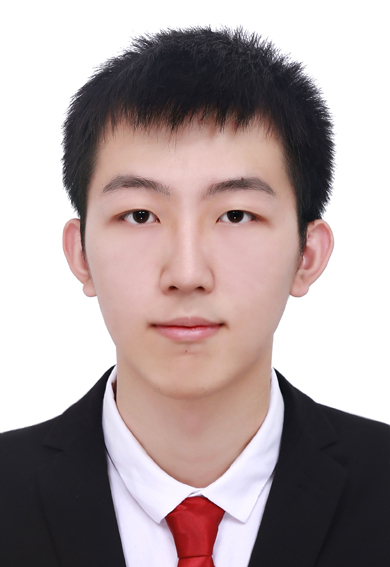}}]{Jiawei Shao}
(Graduate student member, IEEE) received the B.Eng. degree in telecommunication engineering from Beijing University of Posts and Telecommunications in 2019.
He is currently pursuing a Ph.D. degree in the Department of Electronic and Computer Engineering at the Hong Kong University of Science and Technology.
His research interests include edge intelligence and federated learning.
\end{IEEEbiography}
\vskip -2\baselineskip plus -1fil

\begin{IEEEbiography}[{\includegraphics[width=1in,height=1.25in,clip,keepaspectratio]{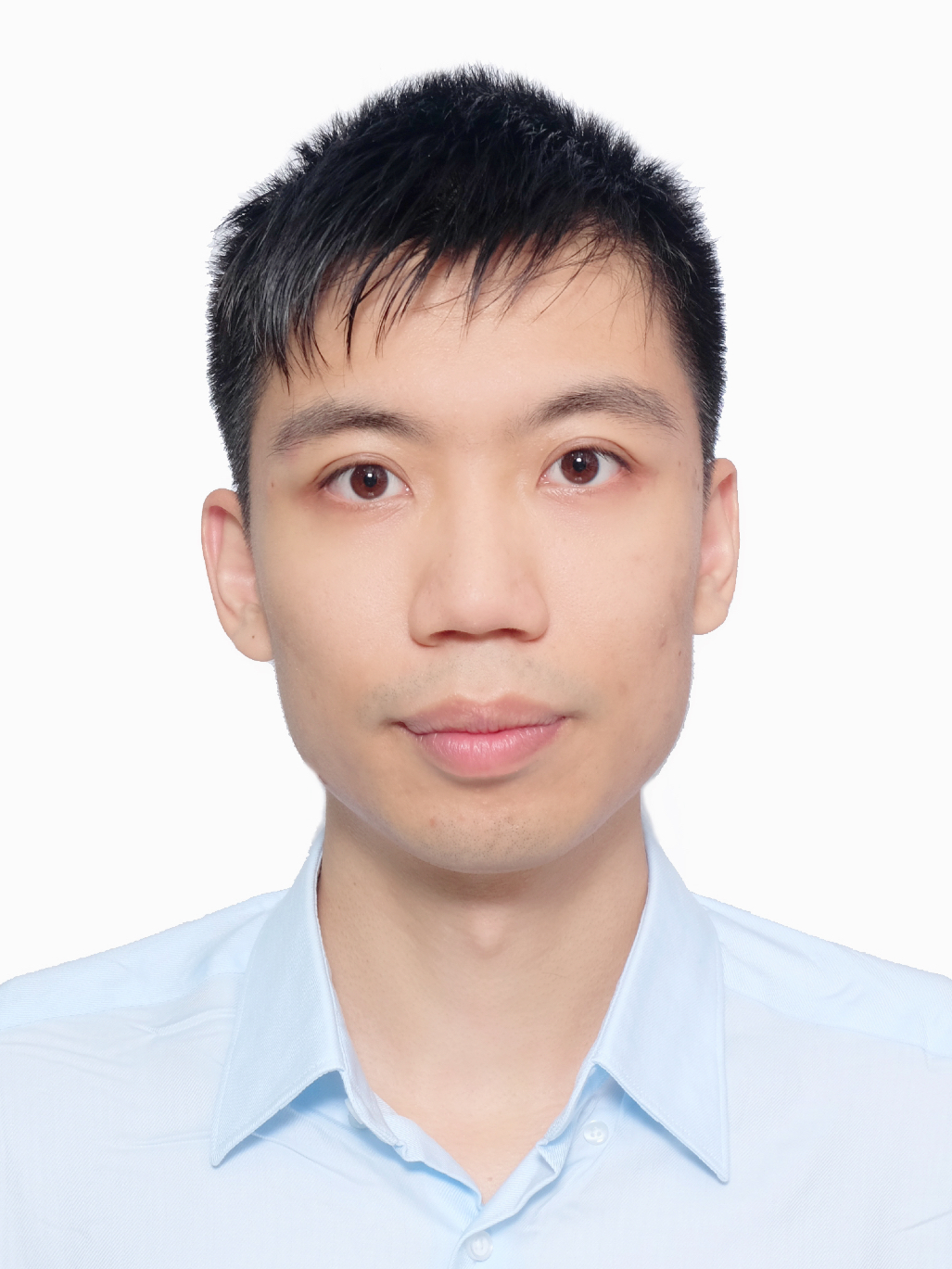}}]{Yuyi Mao}
(Member, IEEE) received the B.Eng. degree in information and communication engineering from Zhejiang University, Hangzhou, China, in 2013, and the Ph.D. degree in electronic and computer engineering from The Hong Kong University of Science and Technology, Hong Kong, in 2017. He was a Lead Engineer with the Hong Kong Applied Science and Technology Research Institute Co., Ltd., Hong Kong, and a Senior Researcher with the Theory Lab, 2012 Labs, Huawei Tech. Investment Co., Ltd., Hong Kong. He is currently a Research Assistant Professor with the Department of Electrical and Electronic, The Hong Kong Polytechnic University, Hong Kong. His research interests include wireless communications and networking, mobile-edge computing and learning, and wireless artificial intelligence.

He was the recipient of the 2021 IEEE Communications Society Best Survey Paper Award and the 2019 IEEE Communications Society and Information Theory Society Joint Paper Award. He was also recognized as an Exemplary Reviewer of the IEEE Wireless Communications Letters in 2021 and 2019 and the IEEE Transactions on Communications in 2020. He is an Associate Editor of the EURASIP Journal on Wireless Communications and Networking.
\end{IEEEbiography}

\newpage

\begin{IEEEbiography}[{\includegraphics[width=1in,height=1.25in,clip,keepaspectratio]{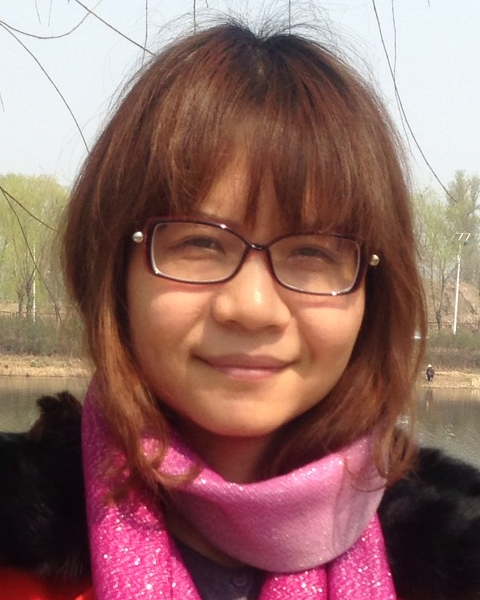}}]{Jessie Hui Wang}
(Member, IEEE) received the B.S. and M.S. degrees in computer science from Tsinghua University and the Ph.D. degree in information engineering from The Chinese University of Hong Kong in 2007. She is currently a Tenured Associate Professor with Tsinghua University. Her research interests include Internet routing, distributed computing, network measurement, and Internet economics.
\end{IEEEbiography}
\vskip -2\baselineskip plus -1fil

\begin{IEEEbiography}[{\includegraphics[width=1in,height=1.25in,clip,keepaspectratio]{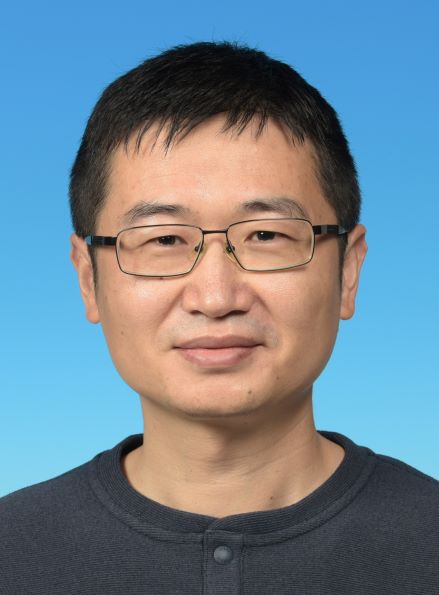}}]{Jun Zhang}

(Fellow, IEEE) received the B.Eng. degree in Electronic Engineering from the University of Science and Technology of China in 2004, the M.Phil. degree in Information Engineering from the Chinese University of Hong Kong in 2006, and the Ph.D. degree in Electrical and Computer Engineering from the University of Texas at Austin in 2009. He is an Associate Professor in the Department of Electronic and Computer Engineering at the Hong Kong University of Science and Technology. His research interests include wireless communications and networking, mobile edge computing and edge AI, and cooperative AI.

Dr. Zhang co-authored the book Fundamentals of LTE (Prentice-Hall, 2010). He is a co-recipient of several best paper awards, including the 2021 Best Survey Paper Award of the IEEE Communications Society, the 2019 IEEE Communications Society \& Information Theory Society Joint Paper Award, and the 2016 Marconi Prize Paper Award in Wireless Communications. Two papers he co-authored received the Young Author Best Paper Award of the IEEE Signal Processing Society in 2016 and 2018, respectively. He also received the 2016 IEEE ComSoc Asia-Pacific Best Young Researcher Award. He is an Editor of IEEE Transactions on Communications, IEEE Transactions on Machine Learning in Communications and Networking, and was an editor of IEEE Transactions on Wireless Communications (2015-2020). He served as a MAC track co-chair for IEEE Wireless Communications and Networking Conference (WCNC) 2011 and a co-chair for the Wireless Communications Symposium of IEEE International Conference on Communications (ICC) 2021. He is an IEEE Fellow and an IEEE ComSoc Distinguished Lecturer.
\end{IEEEbiography}

\clearpage
\appendices

\section{Proof of Proposition \ref{remark1}}
\label{proof_remark_1}
\begin{proof}
    {By expanding the cross-entropy loss of a real data sample $(\bm{x}_k, y_k)$ at client $k$, we have}
        \begin{align}
        &\min_{\bm{w}^e_k, \bm{w}^d_k} - \mathbb{E}_{(\bm{x}_k,y_k) \sim \mathcal{D}_k}  \mathbb{E}_{ \bm{z} \sim p(\bm{z}|\bm{x}_{k}; \bm{w}^e_k)} \log p({y}_k|\bm{z};\bm{w}^d_k) \\
         \equiv &\min_{\bm{w}^{e}_k, \bm{w}^d_k} - \mathbb{E}_{(\bm{x}_k,y_k) \sim \mathcal{D}_k}  \mathbb{E}_{ \bm{z} \sim p(\bm{z}|\bm{x}_{k}; \bm{w}^e_k)} [  \log p( \bm{z}|{y}_k;\bm{w}^d_k) \nonumber \\
         &  + \log p({y}_k ) - \log p(\bm{z} ) ] \\
         \equiv &\min_{\bm{w}^e_k, \bm{w}^d_k} - \mathbb{E}_{(\bm{x}_k,y_k) \sim \mathcal{D}_k}  \mathbb{E}_{ \bm{z} \sim p(\bm{z}|\bm{x}_{k}, \bm{w}^e_k)}  [\log p(\bm{z}|{y}_k;\bm{w}^d_k) \nonumber \\
         & - \log p(\bm{z}|\bm{x}_k; \bm{w}_k^e)] \\
         = & \min_{\bm{w}^e_k, \bm{w}^d_k} D_{\text{KL}} \left( p(\bm{z}|\bm{x}_{k}; \bm{w}^e_k) || p(\bm{z}|{y}_{k};\bm{w}^d_k) \right),
    \end{align}
    where $ p(\bm{z}|{y}_{k}; \bm{w}^d_k)$ is defined as the probability that the intermediate feature inputting into the classifier $\bm{w}_k^d$ is $\bm{z}$ given that it yields a label ${y}_{k}$.
    \end{proof}

\section{Proof of Proposition \ref{remark2}}
\label{proof_remark_2}
\begin{proof}
    {By the cross-entropy loss of client $k$ with a synthetic data sample $(\hat{\bm{x}}_{k^\prime}, y_{k^\prime})$ at client $k$, we have}
        \begin{align}
        &\min_{\bm{w}^e_k, \bm{w}^d_k} - \mathbb{E}_{(\hat{\bm{x}}_{k^\prime},y_{k^\prime}) \sim \hat{\mathcal{D}}_{k^\prime, \text{syn}} } \mathbb{E}_{\bm{z} \sim p(\bm{z}|\hat{\bm{x}}_{k^\prime}; \bm{w}^e_k)} \log p({y}_{k^\prime}|\bm{z};\bm{w}^d_k) \\
         \equiv &\min_{\bm{w}^e_k, \bm{w}^d_k} - \mathbb{E}_{(\hat{\bm{x}}_{k^\prime},y_{k^\prime}) \sim \hat{\mathcal{D}}_{k^\prime, \text{syn}} } \mathbb{E}_{p(\bm{z}|\hat{\bm{x}}_{k^\prime}^i; \bm{w}^e_k)} [  \log p(\bm{z}|{y}_{k^\prime};\bm{w}^d_k) \nonumber \\
         & + \log p({y}_{k^\prime}) - \log p(\bm{z}) ] \\
         \equiv & \min_{\bm{w}^e_k, \bm{w}^d_k} - \mathbb{E}_{(\hat{\bm{x}}_{k^\prime},y_{k^\prime}) \sim \hat{\mathcal{D}}_{k^\prime, \text{syn}} } \mathbb{E}_{p(\bm{z}|\hat{\bm{x}}_{k^\prime}^i; \bm{w}^e_k)} [ \log p(\bm{z}|{y}_{k^\prime}^i;\bm{w}^d_k) \nonumber \\
         & - \log p(\bm{z}|\hat{\bm{x}}_{k^\prime}; \bm{w}_k^e)  ] \\
         = & \min_{\bm{w}^e_k, \bm{w}^d_k} D_{\text{KL}} \left( p(\bm{z}|\hat{\bm{x}}_{k^\prime}^i; \bm{w}^e_k) || p(\bm{z}|{y}_{k^\prime}^i; \bm{w}^d_k) \right) \\
         \label{approx_1}\approx &\min_{\bm{w}^e_k,\bm{w}^d_k} D_{\text{KL}} \left( p(\bm{z}|\hat{\bm{x}}_{k^\prime}^i; \bm{w}^e_g) || p(\bm{z}|{y}_{k^\prime}^i, \bm{w}^d_k) \right) \\
         \label{approx_3}\approx &\min_{\bm{w}^e_k,\bm{w}^d_k} D_{\text{KL}} \left( p(\bm{z}|\Tilde{\bm{x}}_{k^\prime}^i; \bm{w}^e_g) || p(\bm{z}|{y}_{k^\prime}^i, \bm{w}^d_k) \right) \\
         \label{approx_2}\approx &\min_{\bm{w}^e_k,\bm{w}^d_k} D_{\text{KL}} \left( p(\bm{z}|\bm{x}_{k^\prime}^i; \bm{w}^e_g) || p(\bm{z}|{y}_{k^\prime}^i, \bm{w}^d_k) \right),
    \end{align}
    where $\Tilde{\bm{x}}_{k^\prime}^i$ is defined as the synthetic data learned by matching all the features of the real sample without CAM in \eqref{loss_fm}.
    The approximate equation \eqref{approx_1} gives rise to the global model download in each communication round, which becomes an equation when the local update step is set to 1 and the data synthesis is conducted every communication round.
    The approximate equation \eqref{approx_3} is obtained according to the fact that $\hat{\bm{x}}_{k^\prime}^i$ and $\Tilde{\bm{x}}_{k^\prime}^i$ have the same class-relevant features by introducing the CAM for better privacy preservation, and full feature matching data synthesis push the synthetic sample $\Tilde{\bm{x}}_{k^\prime}^i$ to match the feature of the real sample ${\bm{x}}_{k^\prime}^i$, which devotes to the approximate equation \eqref{approx_2}.
    \end{proof}

\section{Proof of Theorem \ref{theorem_convergence}}

We first present two Lemmas that are useful to the proof of Theorem \ref{theorem_convergence}.
\label{proof_theorem1}
\begin{lemma}
\label{hypo_ineq}
    Given the hypothesis spaces $\hat{\mathcal{H}} \coloneqq \{ \hat{\bm{h}} : \mathcal{X} \rightarrow \mathcal{Y} \}$ and $\mathcal{G} \coloneqq \{ \bm{g} : \mathcal{X} \rightarrow \{0,1\}\}$ with $\bm{g}(\bm{x}) = \frac{1}{2} \| \hat{\bm{h}} (\bm{x}) - \hat{\bm{h}}^\prime (\bm{x}) \|_1$, for $\hat{\bm{h}}, \hat{\bm{h}}^\prime \in \hat{\mathcal{H}}$, we have
    \begin{align}
        |\mathcal{L}_\mathcal{D} (\hat{\bm{h}}, \hat{\bm{h}}^\prime) - \mathcal{L}_{\mathcal{D}^\prime }(\hat{\bm{h}}, \hat{\bm{h}}^\prime)  | \leq d_\mathcal{G} (\mathcal{D}, \mathcal{D}^\prime).
    \end{align}
    \begin{proof}
        \begin{align}
            & d_\mathcal{G} (\mathcal{D}, \mathcal{D}^\prime)  \\
            =& 2 \sup_{g\in \mathcal{G}} |  {\Pr}_{\mathcal{D}} [g(\bm{x}=1)] -  {\Pr} _{\mathcal{D}^\prime} [g(\bm{x}=1)]|  \\
            = & \sup_{\hat{\bm{h}},\hat{\bm{h}}^\prime \in \hat{\mathcal{H}} } 2 \left | \frac{1}{2} \mathbb{E}_{\bm{x} \in \mathcal{D}} [\hat{\bm{h}} (x) - \hat{\bm{h}}^\prime (x)] - \frac{1}{2} \mathbb{E}_{\bm{x} \in \mathcal{D}^\prime} [\hat{\bm{h}} (x) - \hat{\bm{h}}^\prime (x)] \right |  \\
            \geq& |\mathcal{L}_\mathcal{D} (\hat{\bm{h}}, \hat{\bm{h}}^\prime) - \mathcal{L}_{\mathcal{D}^\prime }(\hat{\bm{h}}, \hat{\bm{h}}^\prime)  |.
        \end{align}
    \end{proof}
\end{lemma}

\begin{lemma}
\label{lemma2}
    For any $\delta \in (0,1)$, with probability at least $1 - \frac{\delta}{K}$, we have
    \begin{align}
        & \mathcal{L}_{{\mathcal{D}}_k \cup {\mathcal{D}}_{\text{syn}} } (\hat{\bm{h}}_k) \\ 
        \leq & \mathcal{L}_{\hat{\mathcal{D}}_k \cup \hat{\mathcal{D}}_{\text{syn}} } (\hat{\bm{h}}_k) +  \sqrt{-\frac{1}{2}  (\frac{\alpha^2}{m_k} + \frac{(1-\alpha)^2}{m_{\text{syn}}} ) \log \frac{\delta}{2K}}.
        \label{28}
    \end{align}
\end{lemma}
\begin{proof}
    \begin{align}
    &\mathcal{L}_{\hat{\mathcal{D}}_k \cup \hat{\mathcal{D}}_{\text{syn}} } (\hat{\bm{h}}_k) = \alpha \mathcal{L}_{\hat{\mathcal{D}}_k} (\hat{\bm{h}}_k) + (1-\alpha) \mathcal{L}_{\hat{\mathcal{D}}_{\text{syn}}} (\hat{\bm{h}}_k) \\
    =& \frac{1}{m_k + m_{\text{syn}}} \big[\sum_{\bm{x} \in \hat{D}_k} \frac{\alpha (m_k+m_{\text{syn}})}{m_k} \| 
    \hat{\bm{h}}_k (\bm{x}) - \hat{\bm{h}}^* (\bm{x}) \|_1  \nonumber \\
     & \quad+ \sum_{\bm{x} \in \hat{D}_{\text{syn}}} \frac{ (1-\alpha) (m_k+m_{\text{syn}})}{m_{\text{syn}}} \|  \hat{\bm{h}}_k (\bm{x}) - \hat{\bm{h}}^* (\bm{x}) \|_1 \big].
    \end{align}
    Let $X_1^{(k)}, \dots,  X_{m_k}^{(k)}$ and $X_1^{({\text{syn}})}, \dots,  X_{m_k}^{({\text{syn}})}$ be the random variables that take on the values $ \frac{ \alpha (m_k+m_{\text{syn}})}{m_k} \|  \hat{\bm{h}}_k (\bm{x}) - \hat{\bm{h}}^* (\bm{x}) \|$ with $\bm{x} \in \hat{\mathcal{D}}_k$ and $ \frac{ (1-\alpha) (m_k+m_{\text{syn}})}{m_{\text{syn}}} \|  \hat{\bm{h}}_k (\bm{x}) - \hat{\bm{h}}^* (\bm{x}) \|$ with $\bm{x} \in \hat{\mathcal{D}}_{\text{syn}}$, respectively. Please note that the range of $X_1^{(k)}, \dots,  X_{m_k}^{(k)}$ and $X_1^{({\text{syn}})}, \dots,  X_{m_k}^{({\text{syn}})}$ are $[0, \frac{\alpha (m_k + m_{\text{syn}})}{m_k}]$ and $[0, \frac{(1-\alpha) (m_k + m_{\text{syn}})}{m_{\text{syn}}}]$, respectively.

    We represent $\Bar{X}=\mathbb{E}[X]$, $X \in \{ X_1^{(k)}, \dots,  X_{m_k}^{(k)},  X_1^{({\text{syn}})}, \dots,  X_{m_k}^{({\text{syn}})}\}$, which is the empirical loss (i.e., first term) in inequality \eqref{28}.
    Due to the linearity of expectations, we obtain $\mathbb{E}[(\Bar{X})]= \mathcal{L}_{\mathcal{D}_k \cup \mathcal{D}_{\text{syn}}} (\hat{h}_k)$. Then, we employ the Hoeffding’s inequality to show:
    \begin{align}
        &\Pr \big( | \Bar{X} - \mathbb{E} [\Bar{X}] \ge \epsilon | \big) \\
        \leq & 2 \exp \left( - \frac{2 (m_k+ m_{\text{syn}})^2 \epsilon^2}{m_k \left(\frac{\alpha (m_k + m_{\text{syn}})}{m_k} \right)^2 + m_{\text{syn}} \left(\frac{(1-\alpha) (m_k+m_{\text{syn}})}{m_{\text{syn}}} \right)^2 } \right) \\
        =& 2 \exp \left( - \frac{2 \epsilon^2}{ \frac{\alpha^2}{m_k} + \frac{(1-\alpha)^2}{ m_{\text{syn}} } } \right).
    \end{align}
    By defining $\frac{\delta}{K} = 2 \exp \left( - \frac{2 \epsilon^2}{ \frac{\alpha^2}{m_k} + \frac{(1-\alpha)^2}{ m_{\text{syn}} } } \right)$, we have
    \begin{align}
        \epsilon = \sqrt{-\frac{1}{2}  (\frac{\alpha^2}{m_k} + \frac{(1-\alpha)^2}{m_{\text{syn}}} ) \log \frac{\delta}{2K}}.
    \end{align}
\end{proof}

Now we are ready to prove Theorem \ref{theorem_convergence} with the above lemmas.
\begin{proof}
We prove the upper bound for $|\mathcal{L}_{\hat{\mathcal{D}}_g} (\hat{\bm{h}}, \hat{\bm{h}}^*) - \mathcal{L}_{\mathcal{D}_k \cup \mathcal{D}_{\text{syn}}} (\hat{\bm{h}})|$ as follows:
    \begin{align}
        &|\mathcal{L}_{\hat{\mathcal{D}}_g} (\hat{\bm{h}}, \hat{\bm{h}}^*) - \mathcal{L}_{\mathcal{D}_k \cup \mathcal{D}_{\text{syn}}} (\hat{\bm{h}})| \nonumber \\
        = &|\mathcal{L}_{\hat{\mathcal{D}}_g} (\hat{\bm{h}}, \hat{\bm{h}}^*) - \alpha \mathcal{L}_{\mathcal{D}_k} (\hat{\bm{h}}, \hat{\bm{h}}^*) - (1-\alpha) \mathcal{L}_{\mathcal{D}_{\text{syn}}} (\hat{\bm{h}}, \hat{\bm{h}}^*) | \\
        \label{ineq12} \leq &\alpha | \mathcal{L}_{\hat{\mathcal{D}}_g} (\hat{\bm{h}}, \hat{\bm{h}}^*) - \mathcal{L}_{\mathcal{D}_k} (\hat{\bm{h}}, \hat{\bm{h}}^*) | \nonumber \\
        & + (1-\alpha) | \mathcal{L}_{\hat{\mathcal{D}}_g} (\hat{\bm{h}}, \hat{\bm{h}}^*) - \mathcal{L}_{\mathcal{D}_{\text{syn}}} (\hat{\bm{h}}, \hat{\bm{h}}^*) | \\ 
        \label{ineq13} \leq &\alpha | \mathcal{L}_{\hat{\mathcal{D}}_g} (\hat{\bm{h}}^\prime, \hat{\bm{h}}^*) | + \alpha | \mathcal{L}_{\hat{\mathcal{D}}_g} (\hat{\bm{h}}, \hat{\bm{h}}^\prime) - \mathcal{L}_{\mathcal{D}_k} (\hat{\bm{h}}, \hat{\bm{h}}^\prime) | \nonumber \\
         & +  \alpha| \mathcal{L}_{\mathcal{D}_k} (\hat{\bm{h}}^\prime, \hat{\bm{h}}^*) | 
        +  (1-\alpha) | \mathcal{L}_{\hat{\mathcal{D}}_g} (\hat{\bm{h}}^\prime, \hat{\bm{h}}^*)  | \nonumber \\
        &+  (1-\alpha) | \mathcal{L}_{\hat{\mathcal{D}}_g} (\hat{\bm{h}}, \hat{\bm{h}}^\prime) - \mathcal{L}_{\mathcal{D}_{\text{syn}}} (\hat{\bm{h}}, \hat{\bm{h}}^\prime) | \nonumber \\
        & 
        + (1-\alpha) | \mathcal{L}_{\mathcal{D}_{\text{syn}}} (\hat{\bm{h}}^\prime, \hat{\bm{h}}^*) | \\ 
        \label{ineq14} \leq &\alpha [\lambda_k + d_{\mathcal{G}_k} (\mathcal{D}_k, \hat{\mathcal{D}}_g) ] + (1-\alpha) [\lambda_{k,{\text{syn}}} + d_{\mathcal{G}_k} (\mathcal{D}_{\text{syn}}, \hat{\mathcal{D}}_g) ].
    \end{align}
    The inequalities \eqref{ineq12} \eqref{ineq13} follow from the triangle inequality, and the inequality \eqref{ineq14} is obtained from Lemma \ref{hypo_ineq}.

    Additionally, due to the convexity of the risk function and the Jensen's inequality, we obtain
    \begin{align}
        \mathcal{L}_{\hat{\mathcal{D}}_{g}} (\hat{\bm{h}}) \leq \frac{1}{K}\sum_{k=1}^K  \mathcal{L}_{\hat{\mathcal{D}}_{g}} (\hat{\bm{h}}_k, \hat{\bm{h}}^*).
    \end{align}
    Therefore,
    \begin{align}
        &\Pr \big[ \mathcal{L}_{\hat{\mathcal{D}}_{g}} (\hat{\bm{h}}) \nonumber \\
        \geq &\frac{1}{K} \sum_{k=1}^K  [ \mathcal{L}_{\hat{\mathcal{D}}_k \cup \hat{\mathcal{D}}_{\text{syn}} } (\hat{\bm{h}}_k) + \sqrt{-\frac{1}{2}  (\frac{\alpha^2}{m_k} 
        +\frac{(1-\alpha)^2}{m_{\text{syn}}} ) \log \frac{\delta}{2}}
        \nonumber \\
        +& \alpha [ \lambda_k + d_{\mathcal{G}_k} (\mathcal{D}_k, \hat{\mathcal{D}}_g) ]
        + (1-\alpha) [\lambda_{k,\text{syn}} + d_{\mathcal{G}_k} (\mathcal{D}_{\text{syn}}, \hat{\mathcal{D}}_g) ] ] \big] \\
        \leq & \Pr \big[ \frac{1}{K}\sum_{k=1}^K  \mathcal{L}_{\hat{\mathcal{D}}_{g}} (\hat{\bm{h}}_k, \hat{\bm{h}}^*) \geq \frac{1}{K} \sum_{k=1}^K  [ \mathcal{L}_{\hat{\mathcal{D}}_k \cup \hat{\mathcal{D}}_{\text{syn}} } (\hat{\bm{h}}_k) \nonumber \\
        + & \sqrt{-\frac{1}{2}  (\frac{\alpha^2}{m_k} + \frac{(1-\alpha)^2}{m_{\text{syn}}} ) \log \frac{\delta}{2}}
        + \alpha [ \lambda_k + d_{\mathcal{G}_k} (\mathcal{D}_k, \hat{\mathcal{D}}_g) ] \nonumber \\
        & \quad\quad\quad\quad\quad\quad+ (1-\alpha) [\lambda_{k,{\text{syn}}} + d_{\mathcal{G}_k} (\mathcal{D}_{\text{syn}}, \hat{\mathcal{D}}_g) ] ] \big] \\
        \leq & \Pr \big[ \bigcup_{k\in [K]} \mathcal{L}_{\hat{\mathcal{D}}_{g}} (\hat{\bm{h}}_k, \hat{\bm{h}}^*) \geq  [ \mathcal{L}_{\hat{\mathcal{D}}_k \cup \hat{\mathcal{D}}_{\text{syn}} } (\hat{\bm{h}}_k) \nonumber \\
        + & \sqrt{-\frac{1}{2}  (\frac{\alpha^2}{m_k} + \frac{(1-\alpha)^2}{m_{\text{syn}}} ) \log \frac{\delta}{2}}
        + \alpha [ \lambda_k + d_{\mathcal{G}_k} (\mathcal{D}_k, \hat{\mathcal{D}}_g) ] \nonumber \\
        & \quad \quad \quad \quad\quad\quad+ (1-\alpha) [\lambda_{k,{\text{syn}}} + d_{\mathcal{G}_k} (\mathcal{D}_{\text{syn}}, \hat{\mathcal{D}}_g) ] ] \big].
    \end{align}
    According to Lemma \ref{lemma2} and Boole's inequality, we obtain
    \begin{align}
       &\Pr \big[ \mathcal{L}_{\hat{\mathcal{D}}_{g}} (\hat{\bm{h}}) \\
       \geq & \frac{1}{K} \sum_{k=1}^K  [ \mathcal{L}_{\hat{\mathcal{D}}_k \cup \hat{\mathcal{D}}_{\text{syn}} } (\hat{\bm{h}}_k) \nonumber \\
       + & \sqrt{-\frac{1}{2}  (\frac{\alpha^2}{m_k} 
        +\frac{(1-\alpha)^2}{m_{\text{syn}}} ) \log \frac{\delta}{2}}+ \alpha [ \lambda_k + d_{\mathcal{G}_k} (\mathcal{D}_k, \hat{\mathcal{D}}_g) ] \nonumber \\
        +& (1-\alpha) [\lambda_{k,{\text{syn}}} + d_{\mathcal{G}_k} (\mathcal{D}_{\text{syn}}, \hat{\mathcal{D}}_g) ] ] \big] \\
        \leq & \sum_{k \in [K]} \frac{\delta}{K} = \delta.
    \end{align}
\end{proof}

\end{document}